\documentclass[sigconf]{acmart}
\AtBeginDocument{%
  \providecommand\BibTeX{{%
    \normalfont B\kern-0.5em{\scshape i\kern-0.25em b}\kern-0.8em\TeX}}}

\copyrightyear{2024}
\acmYear{2024}
\setcopyright{acmcopyright}
\acmConference[WWW '24] {Proceedings of the ACM Web Conference 2024}{May 13--17, 2024}{Singapore, Singapore.}
\acmBooktitle{Proceedings of the ACM Web Conference 2024 (WWW '24), May 13--17, 2024, Singapore, Singapore}
\acmPrice{15.00}
\acmISBN{979-8-4007-0171-9/24/05}
\acmDOI{10.1145/XXXXXX.XXXXXX}

\usepackage{amsthm}
\usepackage{bbold}
\usepackage{mathtools}
\newtagform{normalsize}[\normalsize]{\normalsize(}{\normalsize)}
\usepackage{adjustbox}
\usepackage{booktabs}
\usepackage{colortbl}
\usepackage{diagbox}
\usepackage{hyperref}
\usepackage{url}
\usepackage{enumitem}
\usepackage{graphicx}
\usepackage{xcolor}
\usepackage{cleveref}
\usepackage{algorithm}
\usepackage{algorithmic}
\usepackage{tcolorbox}
\tcbuselibrary{breakable}
\usepackage{multirow}
\usepackage{makecell}
\usepackage[]{subfig}

\definecolor{grey}{RGB}{180,180,180}
\begin{document}

%%
%% The "title" command has an optional parameter,
%% allowing the author to define a "short title" to be used in page headers.
\title{Masked Graph Autoencoder with Non-discrete Bandwidths}

%%
%% The "author" command and its associated commands are used to define
%% the authors and their affiliations.
%% Of note is the shared affiliation of the first two authors, and the
%% "authornote" and "authornotemark" commands
%% used to denote shared contribution to the research.
\author{Ziwen Zhao}
\orcid{0000-0003-3565-4170}
\affiliation{%
  \department{School of Computer Science and Technology}
  \institution{Huazhong University of Science and Technology}
  \city{Wuhan}
  \country{China}
}
\email{zwzhao@hust.edu.cn}

\author{Yuhua Li}
\orcid{0000-0002-1846-4941}
\authornote{Corresponding author.}
% \email{idcliyuhua@hust.edu.cn}
\author{Yixiong Zou}
\orcid{0000-0002-2125-9041}
% \email{yixiongz@hust.edu.cn}
\affiliation{%
  \department{School of Computer Science and Technology}
  \institution{Huazhong University of Science and Technology}
  \city{Wuhan}
  \country{China}
}
\email{{idcliyuhua,yixiongz}@hust.edu.cn}

\author{Jiliang Tang}
\orcid{0000-0001-7125-3898}
\affiliation{%
  \institution{Michigan State University}
  \city{East Lansing}
  \state{Michigan}
  \country{USA}
}
\email{tangjili@msu.edu}

\author{Ruixuan Li}
\orcid{0000-0002-7791-5511}
\affiliation{%
  \department{School of Computer Science and Technology}
  \institution{Huazhong University of Science and Technology}
  \city{Wuhan}
  \country{China}
}
\email{rxli@hust.edu.cn}

%%
%% By default, the full list of authors will be used in the page
%% headers. Often, this list is too long, and will overlap
%% other information printed in the page headers. This command allows
%% the author to define a more concise list
%% of authors' names for this purpose.
% \renewcommand{\shortauthors}{Trovato and Tobin, et al.}

%%
%% The abstract is a short summary of the work to be presented in the
%% article.
\begin{abstract}
  Masked graph autoencoders have emerged as a powerful graph self-supervised learning method that has yet to be fully explored. In this paper, we unveil that the existing discrete edge masking and binary link reconstruction strategies are insufficient to learn topologically informative representations, from the perspective of message propagation on graph neural networks. These limitations include blocking message flows, vulnerability to over-smoothness, and suboptimal neighborhood discriminability. Inspired by these understandings, we explore non-discrete edge masks, which are sampled from a continuous and dispersive probability distribution instead of the discrete Bernoulli distribution. These masks restrict the amount of output messages for each edge, referred to as ``bandwidths''. We propose a novel, informative, and effective topological masked graph autoencoder using bandwidth masking and a layer-wise bandwidth prediction objective. We demonstrate its powerful graph topological learning ability both theoretically and empirically. Our proposed framework outperforms representative baselines in both self-supervised link prediction (improving the discrete edge reconstructors by at most 20\%) and node classification on numerous datasets, solely with a structure-learning pretext. Our implementation is available at \url{https://github.com/Newiz430/Bandana}. 
\end{abstract}

%%
%% The code below is generated by the tool at http://dl.acm.org/ccs.cfm.
%% Please copy and paste the code instead of the example below.
%%


% \ccsdesc[500]{Do Not Use This Code~Generate the Correct Terms for Your Paper}
% \ccsdesc[300]{Do Not Use This Code~Generate the Correct Terms for Your Paper}
% \ccsdesc{Do Not Use This Code~Generate the Correct Terms for Your Paper}
% \ccsdesc[100]{Do Not Use This Code~Generate the Correct Terms for Your Paper}

%%
%% Keywords. The author(s) should pick words that accurately describe
%% the work being presented. Separate the keywords with commas.

% \received{20 February 2007}
% \received[revised]{12 March 2009}
% \received[accepted]{5 June 2009}

%%
%% This command processes the author and affiliation and title
%% information and builds the first part of the formatted document.
\maketitle

% TL;DR
% We explore a new paradigm of topological masked graph autoencoders with non-discrete masking strategies, named ``bandwidths''. We verify its effectiveness in learning graph topology by both theory and experiment.

\section{Introduction}\label{1}

Today, the demand for massive amounts of data in pre-training large models has reached an unprecedented level. 
{\it Self-supervised learning} (SSL) has emerged as a powerful approach to uncovering underlying patterns in unannotated data by pre-training on some tailor-made tasks called {\it pretexts}~\cite{SSL, GSSL}. 
Graphs, unlike text and images, 
possess non-Euclidean structures that are hard for humans to intuitively understand and lack well-annotated graph benchmark datasets due to the diversity of graph data and tasks.
Thus, SSL also plays a pivotal role in learning graph representations, especially in various web applications such as social recommendation~\cite{CGI}.

\emph{Taxonomy.} 
Contemporary graph SSL studies are mainly {\it contrastive methods}~\cite{DGI, GRACE, GCA} that leverage metric learning between augmented data pairs. 
In spite of this, they suffer from
the thorny problem of dimensional collapse. On the other hand, {\it autoencoding methods} 
learn by reconstructing the input data from encoded representations. However, traditional graph autoencoders~\cite{GAE/VGAE} fall short of modeling high-dimensional representation spaces, while variational autoencoders~\cite{GAE/VGAE, ARGA/ARVGA, SIG-VAE} require additional assumptions on data distributions. 
In contrast, {\it masked graph autoencoders}, a novel framework for data reconstruction, enable the learning of high-dimensional representations without extra assumptions and show remarkable adaptability to graph data. 
One type of masked graph autoencoder aims to reconstruct node features, referred to as \textsc{FeatRec}s~\cite{GMAE, GraphMAE, GraphMAE2}. 
Another type, on which our work focuses, aims to reconstruct randomly masked links to learn graph topology, referred to as \textsc{TopoRec}s~\cite{S2GAE, MaskGAE}. 

\emph{Problem.}
In this work, we focus on the topological informativeness of traditional \textsc{TopoRec}s' representations, i.e. {\it how well they embed graph topology into the latent representation space}. 
As illustrated in Figure \ref{fig:bandwidth}(a), traditional \textsc{TopoRec}s rely on two key components: 
(\lowercase{\romannumeral 1}) {\it discrete edge masking}, where binary edge masks are sampled from a discrete distribution, and (\lowercase{\romannumeral 2}) {\it binary link reconstruction}, which distinguishes the masked positive edges from negative ones. 
Despite some prominent results~\cite{S2GAE, MaskGAE, SeeGera} derive from these two strategies, 
we argue that {\it discrete random masking and binary link reconstruction lead to limited informativeness, both globally and locally} (in Section \ref{3.3}). 
(\lowercase{\romannumeral 1}) Globally, pathways for long-range information are likely to be stretched or blocked by indiscriminate masking, 
leading to the vulnerability to over-smoothing.
(\lowercase{\romannumeral 2}) Locally, 
discrete masking cannot provide fine-grained neighborhood discriminability, leading to suboptimal topological learning performance. 

\begin{figure}
    \centering
    \includegraphics[width=0.96\columnwidth]{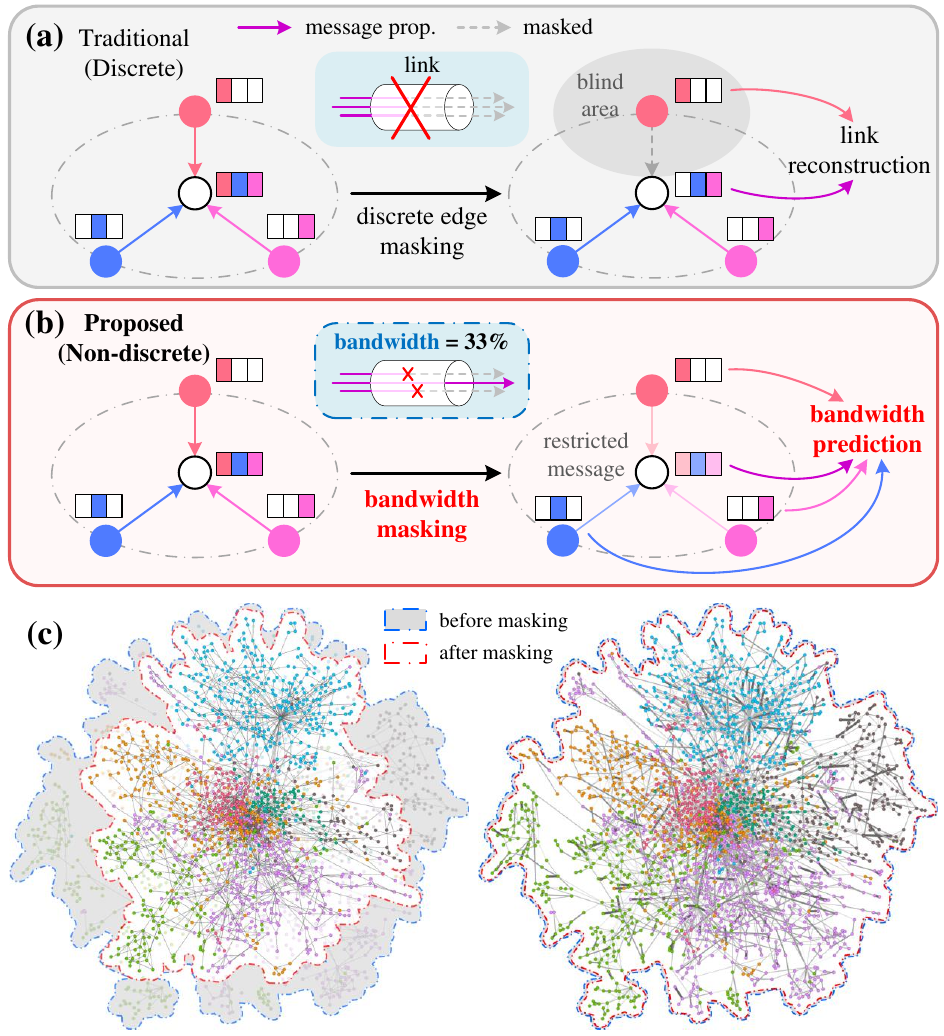}
    \caption{
    Discrete masks vs. the proposed bandwidths. \textmd{{\bf(a)} Traditional \textsc{TopoRec}s randomly mask a fixed proportion of edges and try to reconstruct them. 
    However, messages from some neighboring nodes (e.g. the \textcolor{red}{red} one) as well as their predecessors will not be received by the target node (white). 
    {\bf(b)} 
    We propose bandwidth masking and prediction, which first restricts the message propagated through each edge in varying degrees and then predicts how much it is restricted. The white node can now receive messages from every neighbor.
    {\bf(c)} The connected component of Cora~\protect\cite{Planetoid} \textcolor{blue}{before} and \textcolor{red}{after} different masking schemes. {\bf Left}: discretely masked graph breaks the connectivity of the original component, whereas {\bf right}: bandwidth masked graph (where the width and grayscale of each edge denote the assigned bandwidth) keeps the original graph topology intact, so the reconstructor learns more topologically informative representations.
    Best viewed in color.
    }}
    \label{fig:bandwidth}
\end{figure}

\emph{Present work.}
To tackle the problems, a topologically informative non-discrete masking strategy is desired. We introduce a new perspective by considering the message propagation of a GNN analogously as the (transient) transmission between nodes in a telecommunication network. By randomly setting a limit for each connection link (edge) on the amount of messages transferred in one single propagation step, named the ``{\it bandwidth}'', we manage to {\it mask a portion of information through each edge}, as shown in Figure \ref{fig:bandwidth}(b). Bandwidths provide topological informativeness both globally and locally: 
(\lowercase{\romannumeral 1}) the graph topology is kept intact during the training (Figure \ref{fig:bandwidth}(c)), and long-range information can reach its destination unimpeded, which provides {\it global informativeness}. 
(\lowercase{\romannumeral 2}) By sampling bandwidth values from a dispersive Boltzmann-Gibbs distribution, each node learns its neighborhood in a more fine-grained and discriminative way, which provides {\it local informativeness}. Accordingly, we propose a novel masked graph autoencoder with bandwidth prediction instead of link reconstruction. It is termed ``Graph \underline{a}utoe\underline{n}coder \underline{a}ided by \underline{Band}widths'', \texttt{Bandana} in reverse. We showcase \texttt{Bandana}'s great topological informativeness both theoretically (Section \ref{4.2}) and empirically (Section \ref{5}) by conducting extensive experiments on numerous datasets. 
% for structure learning (link prediction), feature learning (node classification), manifold learning, feasibility of deeper GNNs, etc.
We present the following main contributions:
\begin{itemize}[leftmargin=*]
    \item We unveil that discrete \textsc{TopoRec}s are insufficient to learn topologically informative representations. Globally, {\it blocked message flows} make the \textsc{TopoRec} vulnerable to the over-smoothing problem of deeper GNNs; locally, uniformized weight distribution results in the {\it indiscriminative neighborhood}.
    \item We explore a non-discrete masking mechanism in masked autoencoders, named ``{\it bandwidths}''. We establish a theoretical relationship between our bandwidth mechanism and regularized denoising autoencoders to prove its informativeness.
    \item We propose \texttt{Bandana}, a novel graph self-supervised learning framework that learns topologically informative representations. It outperforms representative baselines in both link prediction and node classification solely with a structure-learning pretext.
\end{itemize}

\section{Related work}\label{2}

In this section, we briefly review prior work on graph self-supervised learning and compare their advantages and disadvantages.

%\emph{Early graph representation learning}. DeepWalk~\cite{DeepWalk} and node2vec~\cite{node2vec} leverage document vectorization in natural language processing. They treat a graph as a sequence of words and encode it into the representation space. However, serializing graphs is unable to fully exploit their structure and feature information. Tailored for graph data, GNNs have developed rapidly for deep graph representation learning in recent years. 

\emph{Contrastive Learning}~\cite{CL} was born initially for the visual domain~\cite{DIM, MoCo}. The most popular contrastive learning framework~\cite{SimCLR} feeds augmented data pairs into two shared-weight neural networks. Then it 
computes pairwise similarities between positive and negative samples 
by InfoNCE contrastive loss~\cite{InfoNCE}.
A flood of prominent work has appeared since contrastive learning was introduced to the graph domain~\cite{GRACE, GCA, CCA-SSG, COSTA}.
%The series have widespread applications in molecular property prediction~\cite{GeomGCL, MoCL, MolCLR}, recommender systems~\cite{CGI, CrossCBR, GCL4SR}, knowledge graph mining~\cite{KGCL, SimKGC}, etc.  
However, they must face the serious problem of representation collapse%~\cite{Align-Uniform, W-MSE, DirectPred, DirectCLR}
, that is, the output of the encoder will degenerate to a scalar independent of the input.
In addition, InfoNCE itself does not provide the power of learning graph structures, because it only measures the distance in the feature space. Therefore, contrastive learning methods rarely discuss the generalizability to structural tasks like link prediction unless they are specifically designed~\cite{T-BGRL}.

\emph{Autoencoding}, another line of work, 
encodes the graph into a latent space via GNNs and then decodes the representations to reconstruct the original features or structure. The pioneering work of GAE~\cite{GAE/VGAE} stimulated the research of traditional graph autoencoders~\cite{ARGA/ARVGA}. However, an overcomplete autoencoder, whose dimension of the representation space is no less than that of the data space, may degenerate into an identity map, 
severely limiting its expressive power. 
Graph variational autoencoders~\cite{GAE/VGAE, ARGA/ARVGA, SIG-VAE, SeeGera} learn prior variational distributions of latent representations to obtain better latent spaces and generate new representations from them. 
Nevertheless, they still induce a suboptimal latent space because of the simplistic prior assumption~\cite{SIG-VAE}. 

\emph{Mask Modeling}. The ``mask-then-reconstruct'' scheme has already been adopted to model natural language, computer vision, etc., and has achieved great success~\cite{BERT, MAE}.
Masked autoencoders (MAEs) have been introduced into the graph domain just recently~\cite{S2GAE}. For \textsc{FeatRec}s, the most well-known of which is the \texttt{GraphMAE} series~\cite{GraphMAE, GraphMAE2} that leverages the re-mask mechanism and the scaled cosine error loss.
Yet, they are not suitable for link prediction due to the neglect of graph topology. 
On the other hand, \textsc{TopoRec}s learn by discrete random edge masks and link reconstruction, the most well-known of which are \texttt{S2GAE} (formerly \texttt{MGAE})~\cite{S2GAE} and \texttt{MaskGAE}~\cite{MaskGAE}. 
\texttt{S2GAE} resorts to a cross-correlation decoder to capture the information lost by perturbation, and \texttt{MaskGAE} employs path masks and another degree regression decoder. 
Despite being empirically beneficial to performance improvement, there are no theoretical guarantees that these strategies can induce a better topological learner.
% Although they have discussed the theoretical connection between \textsc{TopoRec} and contrastive learning, they do not delve into the feasibility of learning graph topology.
% Moreover, despite some existing theoretical frameworks for MAEs in the visual domain~\cite{MAEtheory, MultiView-MRP}, they are currently incompatible with \textsc{TopoRec}s. 
Our work aims to bridge this gap. 

\section{Preliminaries}\label{3}

In this section, we illustrate the principle of message propagation (Section \ref{3.1}) and \textsc{TopoRec}s (Section \ref{3.2}). Then, we discuss the problems of discrete masking and link reconstruction (Section \ref{3.3}).

\subsection{Notations and Concepts}\label{3.1}

We use different types of one specific symbol $S$ to denote different forms of one mathematical object. A {\bf bold} symbol $\mathbf{S}$ denotes a matrix, with its $j$th column in \textbf{\textit{bold italics}} $\boldsymbol{S}_j = \mathbf{S}_{:,j}$. The element at the $i$th row and $j$th column is in {\it italics} with subscripts $S_{ij}$. 

Let $\mathcal{G} = (\mathbf{X}, \mathbf{A})$ be an undirected graph with $n$ nodes, where $\mathbf{X} \in \mathbb{R}^{n \times d}$ is the node feature matrix and $\mathbf{A} \in \{0, 1\}^{n \times n}$ is the adjacency matrix. We also denote by $\mathcal{E}$ the edge set and $\mathcal{V}$ the node set.
$\text{deg}(i)$ is the degree of node $i$. We call any 1-hop subgraph $\mathcal{G}_i = (\mathbf{X}^{\mathcal{G}_{i}}, \mathbf{A}^{\mathcal{G}_{i}}) \subset \mathcal{G}$ of node $i$ an {\it ego-graph}, where $\mathbf{X}^{\mathcal{G}_{i}} = [\boldsymbol{X}_j]_{j\in\mathcal{N}_i \cup \{i\}}$, and $\mathbf{A}^{\mathcal{G}_{i}} \in \{0,1\}^{n_i\times n_i}$ is a principal submatrix of $\mathbf{A}$
Here $n_i=|\mathcal{N}_i \cup \{i\}|$ with $\mathcal{N}_i$ the 1-hop neighborhood set of node $i$.

Message-passing GNNs (MPNNs)~\cite{GCN,GAT} learn by exchanging information between neighboring nodes. To begin with, every node receives messages from its neighbors and processes them with non-linear neurons. Then, messages are aggregated as the new representation of that node. This iterative updating mechanism can be formalized as $\mathbf{Z}^{(k)} \leftarrow \mathbf{G}\mathbf{Z}^{(k-1)}\mathbf{W}^{(k-1)}$, where $\mathbf{W}^{(k)}$ is a learnable weight matrix of the $k$th layer. For notational convenience, we integrate message aggregation and activation into one message propagation matrix $\mathbf{G}$, with the general form $\mathbf{G}=\Sigma\mathbf{A}$ where $\Sigma$ is an activation operator such as Sigmoid $\sigma(\cdot)$ and ReLU. $\mathbf{G}$ varies with GNN types, such as
$\mathbf{G} = \Sigma\hat{\mathbf{D}}^{-1/2}\hat{\mathbf{A}}\hat{\mathbf{D}}^{-1/2}$
for a Graph Convolutional Network (GCN)~\cite{GCN}. Here $\hat{\mathbf{A}}$ represents the graph with self-loops: $\hat{\mathbf{A}}=\mathbf{A}+\mathbf{I}_n$ and $diag(\hat{\mathbf{D}})=\mathbf{1}_n^\top\hat{\mathbf{A}}$, where $\mathbf{I}_n\in\{0, 1\}^{n \times n}$ and $\mathbf{1}_n \in \{1\}^n$ are respectively the identity matrix and the all-ones vector. To sum up, we can denote the output of the $K$th MPNN layer as $\mathbf{Z}^{(K)} = \mathbf{\Gamma}\mathbf{X}\mathbf\Theta$, where $\mathbf{\Gamma}=\mathbf{G}^K$ and $\mathbf\Theta=\prod_{i=0}^{K-1}{\mathbf{W}^{(i)}}$.

\begin{figure}
    \centering
    \includegraphics[width=0.9\columnwidth]{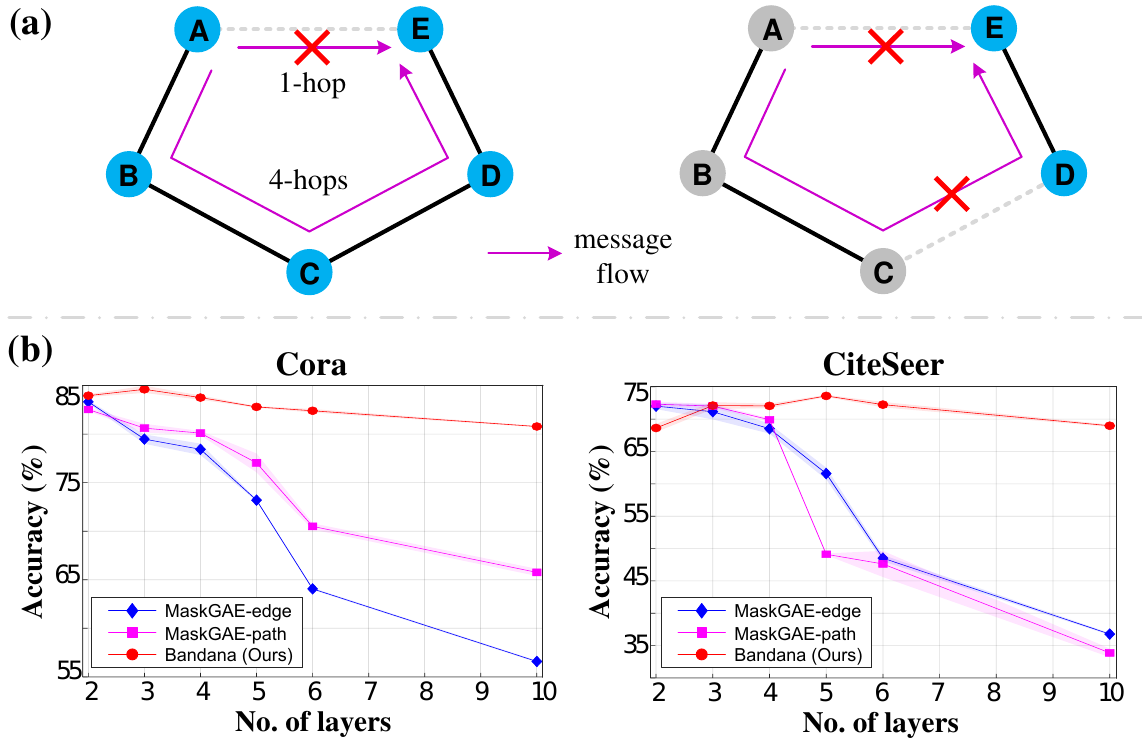}
    \includegraphics[width=\columnwidth]{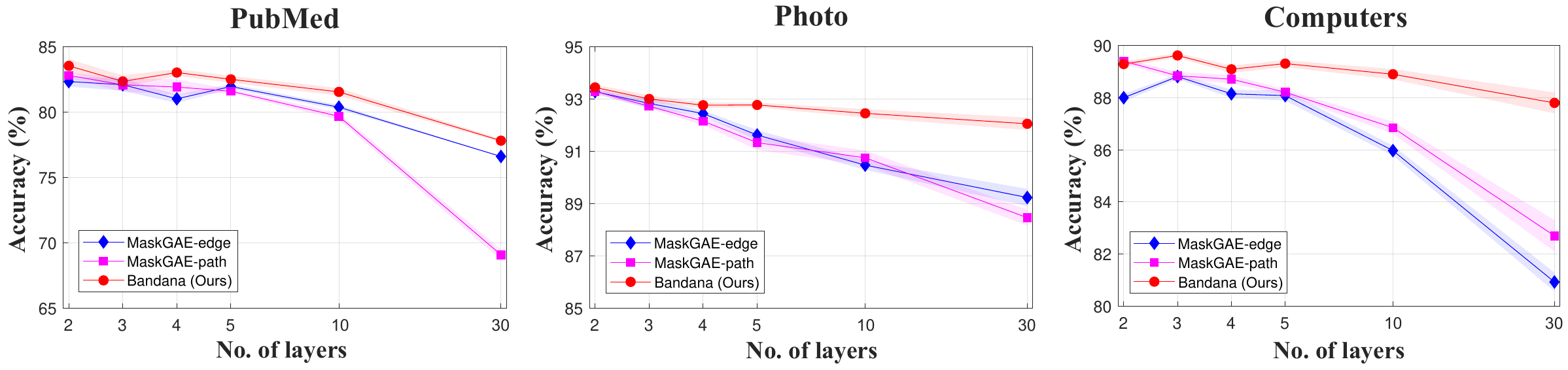}
    \caption{
    Blocked message flows. \textmd{{\bf(a)} A toy example: two paths are available from A to E in the pentagonal graph. {\bf Left}: edge (A,E) is masked out.
    Message from A can only reach E at the cost of being aggregated 3 more times. {\bf Right}: (C,D) is also masked out. E is now out of reach of A (as well as B and C).
    {\bf(b)} Node classification accuracy of a GCN pre-trained by two counterparts of \texttt{MaskGAE} (\textcolor{blue}{blue}, \textcolor{magenta}{magenta}) and \texttt{Bandana} (\textcolor{red}{red}) w.r.t. the network depth. 
    }}
    \label{fig:messageflows}
\end{figure}

\subsection{\textsc{TopoRec}}\label{3.2}

\textsc{TopoRec}s 
mask a subset of edges $\mathcal{E}_m \subset \mathcal{E}$ and use the unmasked set $\bar{\mathcal{E}}_m = \mathcal{E} - \mathcal{E}_m$ along with the entire node set to train an encoder. $\bar{\mathcal{E}}_m$ (resp. $\mathcal{E}_m$) induces a subgraph with adjacency matrix $\bar{\mathbf{A}}_{m}$ (resp. $\mathbf{A}_{m} = \mathbf{A} - \bar{\mathbf{A}}_{m}$). The masking process is defined as
\begin{equation}\label{eq:mask}
\bar{\mathbf{A}}_{m}=\mathbf{A} \circ \mathbf{M}, \quad \mathbf{M} := [\mathbb{1}_{[(i,j)\in \bar{\mathcal{E}}_m]}]^{n \times n} \in \{0, 1\}^{n \times n}
\end{equation}
with $\circ$ the Hadamard product. For discrete \textsc{TopoRec}s, every entry of the masking matrix $\mathbf{M}$ is an indicator $\mathbb{1}_{[(i,j)\in \bar{\mathcal{E}}_m]}$ that determines if the edge $(i,j)$ is retained. It follows an i.i.d. Bernoulli distribution $M_{ij} \sim Bernoulli(1-p), \forall i, j$ where $p$ controls the mask ratio.

Formally, A \textsc{TopoRec} is a parameterized binary map $r(\mathbf{X}, \bar{\mathbf{A}}_{m}): \mathbb{R}^{n \times d} \times \{0, 1\}^{n \times n} \rightarrow (0, 1)^{n \times n}$ implemented by two head-to-tail networks, the so-called {\it encoder-decoder}. Encoder, the GNN to be pre-trained, encodes the input features as latent representations, with the $k$th-layer weights \smash{$\mathbf{W}_{e}^{(k)}$}. Decoder plays an auxiliary role in recovering the representations, with the $k$th-layer weights \smash{$\mathbf{W}_{d}^{(k)}$}.
A \textsc{TopoRec} with a single-layer MLP decoder is formalized as
\usetagform{normalsize}
\begin{equation}\small
r(\mathbf{X}, \bar{\mathbf{A}}_{m}) := \underbracket[0.8pt][3pt]{\Sigma_d(\boldsymbol{b}_d+\mathbf{W}_d}_{\text{decoder}}\underbracket[0.8pt][3pt]{(\mathbf{\Gamma}\mathbf{X}\mathbf\Theta)}_{\text{encoder}}), \ \mathbf{\Gamma}= (\Sigma_e\bar{\mathbf{A}}_{m})^K, \mathbf\Theta=\prod_{i=0}^{K-1}{\mathbf{W}_{e}^{(i)}}
\end{equation}
where $\Sigma_e$ (resp. $\Sigma_d$) is the activation operator of the encoder (resp. decoder). 
Finally, the reconstruction loss $\mathcal{L} = \mathcal{L}(r(\mathbf{X}, \bar{\mathbf{A}}_{m}), \mathbf{A}_{m} )$ minimizes the error between the output and the masked data. The cross-entropy is widely adopted to reconstruct links:
\usetagform{normalsize}
\begin{equation}\small
\label{eq:toporecce}\hspace*{-0.05in}
\mathcal{L}(r(\mathbf{X},\bar{\mathbf{A}}_{m}),\mathbf{A}_{m}):=\boldsymbol{1}_n^\top\left(\frac{\delta_{\mathcal{E}_m}}{|\mathcal{E}_m|} + \frac{\delta_{\mathcal{E}^-}}{|\mathcal{E}^-|}\right)(-\mathbf{A}_{m}\circ\log r(\mathbf{X},\bar{\mathbf{A}}_{m}))\boldsymbol{1}_n
\end{equation}
Here $\scriptsize \delta_{\mathcal{E}} = \delta_{(i,j)}(\mathcal{E}):=\begin{cases}
1, (i,j) \in \mathcal{E} \\
0, (i,j) \notin \mathcal{E}
\end{cases}$
is the Dirac measure: the term \smash{$\scriptsize \frac{\delta_{\mathcal{E}_m}}{|\mathcal{E}_m|}$} in \cref{eq:toporecce} filters and averages every masked edge from the cross-entropy matrix $(-\mathbf{A}_{m}\circ\log r(\mathbf{X},\bar{\mathbf{A}}_{m}))$, while \smash{$\scriptsize \frac{\delta_{\mathcal{E}^-}}{|\mathcal{E}^-|}$} indicates the sampled negative edge set $\mathcal{E}^- \subset \mathcal{V}\times\mathcal{V} - \mathcal{E}$.

\subsection{A Message Propagation View of \textsc{TopoRec}s}\label{3.3}

In this subsection, we revisit the message propagation to answer the question ``{\it why are discrete \textsc{TopoRec}s topologically uninformative?}''.

\subsubsection{Global uninformativeness: blocked message flows.} A {\it message flow} is the directed message pathway from a source node to a sink (target) node~\cite{FlowX}. 
% It has already been an explanation tool for GNN behaviors~\cite{FlowX}. 
Let us analyze the message flows of a discrete \textsc{TopoRec}.
For an ego-graph $\mathcal{G}_i$, the central node $i$ randomly selects a subset of nodes from its neighborhood $\mathcal{N}_i$ and aggregates messages from them only. 
This indiscriminate selection {\it obstructs the message flows that may be crucial to the sink nodes}: 
source nodes of these message flows are not able to transmit their messages directly to the sink nodes, resulting in a large amount of stretched or even blocked flows, which are very likely to disrupt the connectivity of the original graph (as the mask ratios of these masked autoencoders are usually very high~\cite{MAE, MaskGAE}), as shown in Figure \ref{fig:messageflows}(a).

\begin{figure}
    \centering
    \includegraphics[width=\columnwidth]{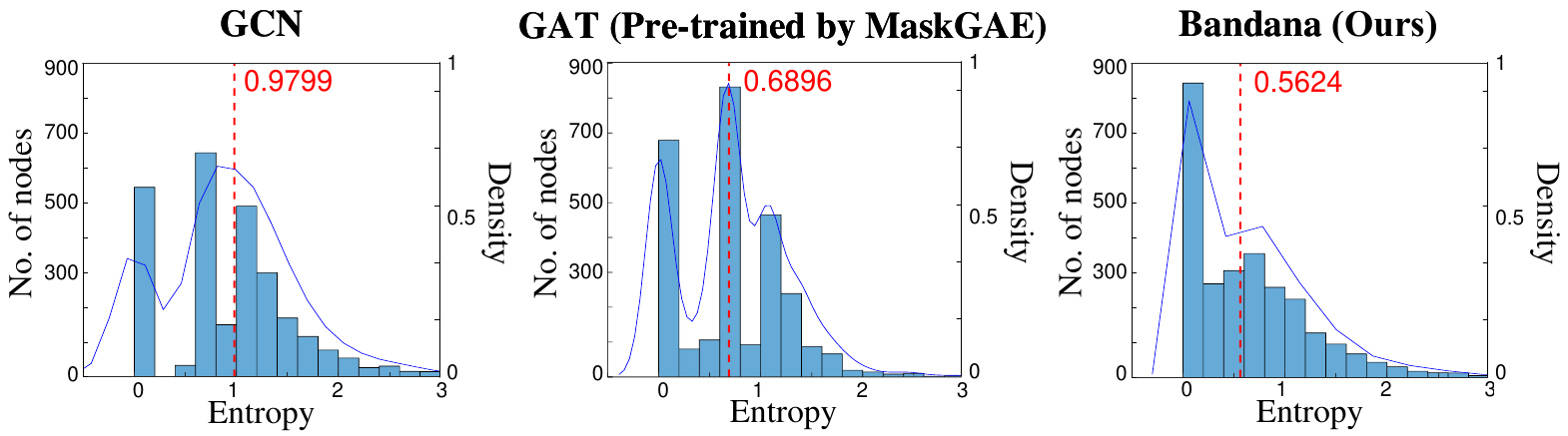}
    \caption{
    Entropy histograms of the edge weight distribution in ego-graphs on Cora. \textmd{\textcolor{blue}{Blue} solid lines are the Gaussian kernel density estimation curves with the \textcolor{red}{red} dashed lines medians. 
    }}
    \label{fig:neighborentropy}
\end{figure}

Moreover, we reveal that {\it discrete masking makes the encoder vulnerable to the over-smoothing problem}. 
To formalize, we introduce a commonly used metric, the {\it Dirichlet energy}~\cite{EGNN, OverSmooth}, to evaluate the over-smoothness of a discretely masked graph.  
The lower the energy, the severer the over-smoothness. Our conclusion is summarized by the following theorem.

\begin{tcolorbox}[breakable, pad at break=2.5mm, before skip=1mm, after skip=1mm, width=1.0\linewidth, boxsep=0mm, arc=1.5mm, left=1.5mm, right=1.5mm, top=1.5mm, bottom=1mm]
\begin{theorem}[Vulnerability of discrete \textsc{TopoRec}s to over-smoothing] \label{thm:dirichlet}
Let $\mathcal{G}_i=(\mathbf{X}^{\mathcal{G}_{i}}, \mathbf{A}^{\mathcal{G}_{i}})$ be an ego-graph with $n_i\ge2$. Assume $\boldsymbol{X}_j^{\mathcal{G}_{i}} = \boldsymbol{X}_k^{\mathcal{G}_{i}}$ for $\forall j,k \in \mathcal{N}_i$. Define the \emph{ego Dirichlet Energy} of $\mathcal{G}_i$ as 
\begin{equation}
E_D(\mathcal{G}_i) := \frac{1}{n_i}\sum_{j \in \mathcal{N}_i}\|\boldsymbol{X}_i^{\mathcal{G}_{i}}-\boldsymbol{X}_j^{\mathcal{G}_{i}}\|^2
\end{equation}
If a connected component $\mathcal{G}_{i,m}$ of $\mathcal{G}_i$ is induced by imposing masks following the i.i.d. $Bernoulli(1-p), 0<p \le 1$ to $\mathbf{A}^{\mathcal{G}_{i}}$, then we have $E_D(\mathcal{G}_{i,m}) \le E_D(\mathcal{G}_i)$. This inequality is an equality iff $p=1$.
\end{theorem}
\begin{proof} Please refer to Appendix \ref{A.1}.
\end{proof}
\end{tcolorbox}

Note that the Dirichlet energy of the entire graph is exactly the sum of ego Dirichlet energies over all ego-graphs. So 
Theorem \ref{thm:dirichlet} remarks that a discretely masked graph is more likely to be over-smoothed. As such, discrete \textsc{TopoRec}s obtain relatively trivial expressive power with deeper GNN layers. We have also conducted an experiment to verify this on 5 graph benchmark datasets. It is shown in Figure \ref{fig:messageflows}(b) that the performance of \texttt{MaskGAE} goes into a nosedive with a 5-or-more-layer GCN.

\subsubsection{Local uninformativeness: indiscriminative neighborhood.} A GCN or Graph Attention Network (GAT)~\cite{GAT} is usually chosen as the encoder of a \textsc{TopoRec}. However, both of them are {\it not capable of distinguishing the messages among different neighbors effectively}. We quantify the discriminability of each neighbor $j$'s message in an ego-graph $\mathcal{G}_i$ as {\it the dispersion of edge weights} assigned by node $i$. In GCN, such assignment is realized only by a function of node degrees $w_{ij}=1/\sqrt{(\deg(i)+1)(\deg(j)+1)}$. 
This does not work well due to the power-law tendency of networks~\cite{scale-free}.
GAT, by contrast, explicitly models the neighborhood by introducing learnable self-attention matrices: $w_{ij}=\text{softmax}_j({\text{LeakyReLU}(\mathbf{a}^\top[\mathbf{W}_{att}\boldsymbol{Z}_i||\mathbf{W}_{att}\boldsymbol{Z}_j])})$, but its discriminability is limited either, as previous research~\cite{StatGAT} indicates that the attention weights assigned by GAT are roughly the same for different neighbors. 
Unfortunately, the discrete masking strategy will not provide such neighbor discriminability, because their only concern is the {\it existence} of links to {\it some} neighbors. 

We have conducted another experiment on Cora for demonstration. Edge weights assigned by GCN and GAT pre-trained by \texttt{MaskGAE} are first computed. Then, {\it entropies} of the weights in every ego-graph $H_i = -\sum_{j \in \mathcal{N}_i}w_{ij} \log {w_{ij}}$ are calculated. The smaller the entropy value, the more discriminative the edge weights. We visualize the entropies in Figure \ref{fig:neighborentropy} ({\bf left}, {\bf center}) and observe that 
for both GCN and GAT, the entropy distribution peaks shift towards the larger area, since \texttt{MaskGAE} is not able to provide discriminative power of neighbor messages for them.

To sum up, a new masking scheme is desired instead of discrete edge masking for better topological informativeness.

\section{Bandana}\label{4}

In this section, we aim to answer the questions ``{\it what is \texttt{Bandana}?}'' in Section \ref{4.1} and ``{\it why does \texttt{Bandana} learn informative representations?}'' in Section \ref{4.2}. \texttt{Bandana} is a novel topological masked graph autoencoder that encompasses two main mechanisms: {\it continuous bandwidth masks} and {\it layer-wise bandwidth prediction}.

\subsection{Bandwidth Masking and Prediction Pipeline}\label{4.1}

\subsubsection{Continuous bandwidth masks.}
According to Section \ref{3.2}, a discrete \textsc{TopoRec} samples edge masks from a Bernoulli distribution. As bandwidths are non-discrete, each entry of the masking matrix $\mathbf{M}$ should be randomly sampled from a continuous distribution instead, with the following requirements:

\begin{enumerate}[leftmargin=0.3in]
    \item[(\lowercase{\romannumeral 1})] 
    {\it Probabilistic}: $M_{ij} \in [0, 1]$ (0 for non-existent edges).
    \item[(\lowercase{\romannumeral 2})] 
    {\it Simplicial}: $\sum_{j\in\mathcal{N}_i}{M_{ij}} = 1, \sum_{j\notin\mathcal{N}_i}{M_{ij}} = 0$.
    \item[(\lowercase{\romannumeral 3})] 
    {\it Dispersive}: bandwidths of neighbors should be discriminative.
\end{enumerate}

For (\lowercase{\romannumeral 1}) and (\lowercase{\romannumeral 2}), every column of $\mathbf{M}$, i.e., bandwidths of neighbors in an ego-graph, should form a probabilistic simplex to stabilize the message passing process.
For (\lowercase{\romannumeral 3}), we want bandwidths to provide the discriminative power of neighbors. In light of these conditions, we choose the bandwidth distribution as follows.

\begin{tcolorbox}[breakable, pad at break=2.5mm, before skip=1mm, after skip=1mm, width=1.0\linewidth, boxsep=0mm, arc=1.5mm, left=1.5mm, right=1.5mm, top=1.5mm, bottom=1mm]
\begin{definition}[Bandwidth]\label{def:bandwidth}
For the rest of the paper, $\mathbf{M}\in [0, 1]^{n \times n}$ is a \emph{continuous matrix} consisting of i.i.d. probabilistic simplicial column vectors, of which each nonzero entry follows a \emph{Boltzmann-Gibbs distribution}: 
\begin{equation}\label{eq:bandwidth}
M_{ij} = \text{softmax}_j\left(\frac{m_{ij}}{\tau}\right) = \displaystyle\frac{\exp(m_{ij}/\tau)}{\sum_{k \in \mathcal{N}_j \cup \{j\}}{\exp(m_{kj}/\tau)}}
\end{equation}
where $m_{ij} \sim \mathcal{N}(0, 1)$ and $\tau$ denotes the temperature.
\end{definition}
\end{tcolorbox}

Intuitively, a ``bandwidth'' on an edge is the maximum proportion of the output messages to the input messages through that edge per message passing step. The softmax in \cref{eq:bandwidth} plays a dual role of {\it normalization} and {\it amplification}: (\lowercase{\romannumeral 1}) normalization guarantees a probabilistic simplex; (\lowercase{\romannumeral 2}) exponential softmax amplifies the weight dispersion in an ego-graph, which has already been discovered and utilized by some attention-based studies~\cite{cosFormer}. 
Though both are edge weights, bandwidths are fundamentally different from attention weights, which is further discussed in Appendix \ref{C.1}.

Instead of the discrete masking matrix, we use bandwidths to perturb the adjacency matrix: $\tilde{\mathbf{A}}=\mathbf{A} \circ \mathbf{M}$. This converts the initial discrete adjacency matrix with $A_{ij} \in \{0, 1\}$ into a continuous matrix with $\tilde{A}_{ij} \in [0, 1]$. Unlike the discrete case, both $\mathbf{M}$ and $\tilde{\mathbf{A}}$ are no longer symmetric, because there are two different bandwidth values on every edge for the bidirectional message propagation.

Note that we also adopt the temperature $\tau$ to control the bandwidth distribution. Specifically, it controls the {\it continuity of the mask}: when $\tau\rightarrow0$, the Boltzmann-Gibbs distribution degenerates to the superposition of Dirac $\delta$ functions at 0 and 1, that is, the discrete Bernoulli mask; when $\tau\rightarrow\infty$, it degenerates to Uniform. 

\subsubsection{Encoding.} \texttt{Bandana}'s encoder network propagates bandwidth-restricted messages. To be more specific, the perturbed adjacency matrix $\tilde{\mathbf{A}}$ represents an undirected graph with bidirectional edge weights, on which \texttt{Bandana} performs message propagation instead of $\bar{\mathbf{A}}_{m}$. Propagation on the $k$th layer can be formalized as
\begin{equation}\label{eq:bandanaprop}
\mathbf{Z}^{(k)} \leftarrow \tilde{\mathbf{G}}\mathbf{Z}^{(k-1)}\mathbf{W}_{e}^{(k-1)}, \quad \tilde{\mathbf{G}}=\Sigma_e\tilde{\mathbf{A}}
\end{equation}
The entire encoder-decoder is defined as
\begin{equation}
r(\mathbf{X}, \mathbf{A}) := \underbracket[0.8pt][3pt]{\Sigma_d(\boldsymbol{b}_d+\mathbf{W}_d}_{\text{decoder}}\underbracket[0.8pt][3pt]{(\tilde{\mathbf{\Gamma}}\mathbf{X}\mathbf\Theta)}_{\text{encoder}}), \quad \tilde{\mathbf{\Gamma}} = \tilde{\mathbf{G}}^{K}
\end{equation}
where $\Sigma_d$ refers to the softmax function. It now models the representation space in a way that every node receives and aggregates different ratios of messages from different neighbors.

\subsubsection{Bandwidth prediction.}
Following the asymmetric encoder-decoder architecture~\cite{MAE}, \texttt{Bandana} employs a lightweight MLP as its decoder. What the bandwidth decoder ``reconstructs'' is the bandwidth value of every edge, i.e. it {\it predicts how much every edge is masked during training}. 
This is a logistic regression problem 
that can still be optimized by the cross-entropy objective:
\begin{equation}\label{eq:bandce}
\mathcal{L}(r(\mathbf{X}, \mathbf{A}), \tilde{\mathbf{A}} ):=\boldsymbol{1}_n^\top\left(\frac{\delta_{\mathcal{E}}}{|\mathcal{E}|}+ \frac{\delta_{\mathcal{E}^-}}{|\mathcal{E}^-|}\right)(-\tilde{\mathbf{A}}\circ\log r(\mathbf{X}, \mathbf{A}))\boldsymbol{1}_n
\end{equation}
The difference is that all positive edges ($\delta_{\mathcal{E}}$) are now participating in training. \
We still keep blocked edges $\delta_{\mathcal{E}^-}$ as zero samples.

\subsubsection{Layer-wise masking and prediction.} It has become common knowledge that different network layers capture different granularities of information: shallower layers capture more general information, while deeper layers capture information more specific to the pretext task
~\cite{diff-layers1}.%diff-layers2, diff-layers3
We propose a layer-wise masking scheme to explicitly capture different granularities. We generate different bandwidth masks for every layer of GNN, with the $k$th-layer perturbed adjacency matrix $\tilde{\mathbf{A}}^{(k)}$ and the corresponding message propagation matrix $\tilde{\mathbf{G}}^{(k)}$. On the backend, we share one MLP decoder for every layer.
We calculate the reconstruction loss for each layer and the final loss is the average of all layer losses:
\usetagform{normalsize}
\begin{equation}\small
\mathcal{L} = \frac{1}{K}\sum_{k=0}^{K-1}{\mathcal{L}^{(k)}} = \frac{1}{K}\sum_{k=0}^{K-1}{\mathcal{L}( \Sigma_d(\boldsymbol{b}_d+\mathbf{W}_d(\tilde{\mathbf{\Gamma}}^{(k)}\mathbf{X}\mathbf\Theta^{(k)})), \tilde{\mathbf{A}}^{(k)})}
\end{equation}
where \smash{$\tilde{\mathbf{\Gamma}}^{(i)} = \prod_{i=0}^{k-1}{\tilde{\mathbf{G}}^{(k)}}, \mathbf\Theta^{(k)}=\prod_{i=0}^{k-1}{\mathbf{W}_{e}^{(i)}}$}. 

\subsection{Why Are Bandwidths Informative?}\label{4.2}

In this subsection, we give empirical and theoretical support for our bandwidth schemes.

\subsubsection{A fine-grained strategy for informative topology.}\label{4.2.1}
The advantages of bandwidth masking and prediction are threefold. 
\begin{itemize}[leftmargin=*]
    \item Compared with binary link reconstruction, predicting a continuous bandwidth value is a more {\it fine-grained and challenging} task which is more meaningful to the mask modeling~\cite{MAE}.
    \item {\it Global informativeness}, as non-discrete bandwidth masks guarantee a complete graph topology and unimpeded message flows so that deeper GNNs can be pre-trained more effectively. As shown in Figure \ref{fig:messageflows}(b), \texttt{Bandana} greatly outperforms \texttt{MaskGAE} on GNN pre-training with 5 or more layers.
    \item {\it Local informativeness}, as \texttt{Bandana} can provide the discriminative power of neighbor messages by bandwidth prediction. 
    As shown in Figure \ref{fig:neighborentropy} ({\bf right}), \texttt{Bandana} has the most discriminative neighborhood weights. 
\end{itemize} 

Such informativeness is further validated by an embedding visualization study on Karate Club~\cite{KarateClub} in Appendix~\ref{B}.

\subsubsection{Implicit optimization on the topological manifold.}
\texttt{Bandana}'s excellent topological learning ability is theoretically guaranteed. We elucidate that \texttt{Bandana} can be interpreted as a regularized denoising autoencoder~\cite{DAE} in an implicit graph topological space, while a discrete \textsc{TopoRec} cannot. Furthermore, bandwidth prediction is mathematically equivalent to optimizing a ``score'' in that space. 
To this end, we first assign each column of the adjacency matrix $\mathbf{A}$ to the corresponding node in the graph as its new ``feature''.

\begin{tcolorbox}[breakable, pad at break=2.5mm, before skip=1mm, after skip=1mm, width=1.0\linewidth, boxsep=0mm, arc=1.5mm, left=1.5mm, right=1.5mm, top=1.5mm, bottom=1mm]
\begin{definition}[Topological encoding]
    A \emph{topological encoding matrix} is defined as $\mathbf{T}:=\mathbf{A} - \boldsymbol{1}_n\boldsymbol{1}_n^\top \in \mathbb{R}^{n \times n}$, where $\boldsymbol{1}_n$ is the all-ones vector. Denote $\boldsymbol{T}_j$ as the \emph{topological encoding} of node $j$.
\end{definition}
\end{tcolorbox}

One advantage of the topological encoding is that it allows us to write the bandwidth masking as adding random noises on it:
\begin{equation}
\tilde{\mathbf{T}} = \delta_{\mathcal{E}}(\mathbf{T} + \boldsymbol{\epsilon})=\mathbf{T} + \boldsymbol{\epsilon}, \ \boldsymbol{\epsilon}_i \sim \text{softmax}(\mathcal{N}(0, \mathbf{I}_n))
\end{equation}

Assume \smash{$\boldsymbol{T}_j$} and the perturbed \smash{$\tilde{\boldsymbol{T}}_j$} follow the probability distributions \smash{$p(\boldsymbol{T}_j)$} and \smash{$p(\tilde{\boldsymbol{T}_j})$} respectively. 
From the topological encoding perspective, $\mathbf{X}$ and $\mathbf{\Theta}$ are conversely the non-linear transformations on $\tilde{\mathbf{A}}$, in which case we denote our bandwidth reconstructor by $r_{\mathbf{X}}$. Under this premise, the following Proposition gives that $r_{\mathbf{X}}$ can be viewed as a regularized denoising autoencoder.

\begin{tcolorbox}[breakable, pad at break=2.5mm, before skip=1mm, after skip=1mm, width=1.0\linewidth, boxsep=0mm, arc=1.5mm, left=1.5mm, right=1.5mm, top=1.5mm, bottom=1mm]
\begin{proposition}[Non-discrete \textsc{TopoRec} is a denoising autoencoder]\label{thm:dae}
Suppose a \textsc{TopoRec} on vectors $r_{\mathbf{X}}: \mathbb{R}^{n} \rightarrow \mathbb{R}^{n}$ is at least first-order differentiable (to $\boldsymbol{T}_{j}$ for the rest of the paper). If the perturbed topological encoding $\tilde{\boldsymbol{T}}_j$ on a connected graph follows a continuous distribution $p(\tilde{\boldsymbol{T}_j})$ and satisfies
\begin{enumerate}[leftmargin=0.3in]
    \item[(\lowercase{\romannumeral 1})] 
    $n \ll 2|\mathcal{E}|$, and
    \item[(\lowercase{\romannumeral 2})] 
    all elements in $\{\tilde{\boldsymbol{T}}_{j}\}_{j=1}^{n}$ follow an i.i.d. \emph{isotropic multivariate Gaussian} \smash{$\mathcal{N}(\boldsymbol{\mu}_{\tilde{\boldsymbol{T}}_{j}},\mathbf{\Sigma}_{\tilde{\boldsymbol{T}}_{j}})$}, i.e. the covariance matrix satisfies \smash{$\mathbf{\Sigma}_{\tilde{\boldsymbol{T}}_{j}}=c\mathbf{I}$} where $c$ is an arbitrary constant.
\end{enumerate}
Then, $\mathcal{L}$ in \cref{eq:bandce} defines a regularized denoising autoencoder:
\begin{equation}\small
\label{eq:rdae}\hspace*{-0.05in}
\mathcal{L} = \underbracket[0.8pt][2pt]{\mathbb{E}_{j\in\mathcal{V}}[\| r_{\mathbf{X}}(\boldsymbol{T}_{j}) - \boldsymbol{T}_{j} \|^2]}_\text{reconstruction} + \underbracket[0.8pt][2pt]{\sigma_{\epsilon}^2 \mathbb{E}_{j\in\mathcal{V}}[\| \nabla r_{\mathbf{X}}(\boldsymbol{T}_{j})\|^2_F]}_\text{regularization} +o(\sigma_{\epsilon}^2)
\end{equation}
where $\sigma_{\epsilon}^2$ is the noise variance.
\end{proposition}
\begin{proof}
Please refer to Appendix \ref{A.2}. We also discuss the mildness of the assumptions in Appendix \ref{A.4}. 
\end{proof}
\end{tcolorbox}

Proposition \ref{thm:dae} is consistent with previous studies that the masked autoencoder is a kind of denoising autoencoder~\cite{MAE}, but the noise should technically be non-discrete for \textsc{TopoRec}s. 
According to the existing analysis of the denoising autoencoder~\cite{RAETheory1}, we have the following theorem.

\begin{tcolorbox}[breakable, pad at break=2.5mm, before skip=1mm, after skip=1mm, width=1.0\linewidth, boxsep=0mm, arc=1.5mm, left=1.5mm, right=1.5mm, top=1.5mm, bottom=1mm]
\begin{theorem}[Bandwidth prediction optimizes in the topological encoding space]\label{thm:gradopt}
Suppose $r_{\mathbf{X}}$ fulfills the condition in Proposition \ref{thm:dae}. Then for $\sigma_{\epsilon}^2 \rightarrow 0$, the optimal \textsc{TopoRec} $r_{\mathbf{X}}^* = \arg\min_{r_{\mathbf{X}}}{\mathcal{L}}$ is asymptotically equivalent to an implicit gradient optimizer of $\log p(\boldsymbol{T}_{j})$:
\begin{equation}
    r^*_{\mathbf{X}}(\boldsymbol{T}_{j}) - \boldsymbol{T}_{j}  \propto \nabla \log{p(\boldsymbol{T}_{j})}
\end{equation}
\end{theorem}
\begin{proof}
Please refer to Appendix \ref{A.3}.
\end{proof}
\end{tcolorbox}

By Theorem \ref{thm:gradopt}, the optimal bandwidth predictor ($r_{\mathbf{X}}=r_{\mathbf{X}}^*$) optimizes the gradient of log probabilities of the topological encoding, indicating that the bandwidth masking and prediction scheme are {\it theoretically learning graph topology}.  
Based on our conclusions, \texttt{Bandana} can be further expanded to some theoretically grounded frameworks, such as score-based models~\cite{NCSN} and energy-based models~\cite{EBM}. We have further discussions in Appendix \ref{C.2}.

\section{Experiment Analyses}\label{5}

In this section, we first introduce experimental configurations in Section \ref{5.1}. More detailed settings can be found in Appendix \ref{D}. Then, we showcase the experiment results of \texttt{Bandana} to answer the following research questions:
\begin{itemize}[leftmargin=*]
    \item RQ1. {\it Is \texttt{Bandana} able to learn more informative topology than discrete \textsc{TopoRec}s in practice?} 
    \item RQ2. {\it How does \texttt{Bandana} perform on node classification?} 
    \item RQ3. {\it How does \texttt{Bandana} perform on link prediction?} 
    \item RQ4. {\it How effective are Boltzmann-Gibbs bandwidths and the layer-wise strategy?} 
\end{itemize}

More experiment results can be found in Appendix \ref{E}, including time and space consumptions (\ref{E.1}), more large-scale datasets (\ref{E.3}), parameter analyses (\ref{E.5}), the semi-supervised setting (\ref{E.6}), etc.

\subsection{Experimental Settings}\label{5.1}

\emph{Datasets.} Apart from the two synthetic datasets in Section \ref{5.2}, we conduct experiments on 12 well-known undirected and unweighted graph benchmark datasets, including (\lowercase{\romannumeral 1}) citation networks: Cora, CiteSeer, PubMed~\cite{Planetoid}; (\lowercase{\romannumeral 2}) co-purchase networks: Photo, Computers~\cite{Amazon/Coauthor}; (\lowercase{\romannumeral 3}) co-author networks: CS, Physics~\cite{Amazon/Coauthor}; (\lowercase{\romannumeral 4}) OGB networks: ogbn-arxiv, ogbn-mag, ogbl-collab, ogbl-ppa~\cite{OGB}; (\lowercase{\romannumeral 5}) hyperlink networks: Wiki-CS~\cite{Wiki-CS}. 
Detailed statistics can be found in Appendix \ref{D.1}, and experiment results of the last 4 datasets can be found in Appendix \ref{E.3} and \ref{E.4}.

\emph{Reproducibility.}
We report all quantitative results as ``mean ± standard deviation'' by running 10 times under the same setup. Hardware, training setups, and hyperparameters can be found in Appendix \ref{D.2} and \ref{D.3}.

\emph{Baselines.} As self-supervised methods are being studied, only self-supervised algorithms are considered as baselines. 
They are divided into the following categories:
(\lowercase{\romannumeral 1}) traditional graph autoencoders: \texttt{GAE}~\cite{GAE/VGAE}, \texttt{ARGA}~\cite{ARGA/ARVGA}; (\lowercase{\romannumeral 2}) variational graph autoencoders: \texttt{VGAE}~\cite{GAE/VGAE}, \texttt{ARVGA}~\cite{ARGA/ARVGA}, \texttt{SIG-VAE}~\cite{SIG-VAE}, \texttt{SeeGera}~\cite{SeeGera}\footnote{\texttt{SeeGera} fuses mask modeling with the variational autoencoder. We still count it as variational-based in light of its generative characteristic and learning objective.}; (\lowercase{\romannumeral 3}) contrastive and non-contrastive (with no negative sampling) methods: \texttt{GRACE} \cite{GRACE}, \texttt{GCA} \cite{GCA}, \texttt{COSTA} \cite{COSTA}, \texttt{CCA-SSG} \cite{CCA-SSG}, \texttt{T-BGRL} \cite{T-BGRL}; (\lowercase{\romannumeral 4}) \textsc{FeatRec}s: \texttt{GraphMAE} \cite{GraphMAE}, \texttt{GraphMAE2} \cite{GraphMAE2}; (\lowercase{\romannumeral 5}) \textsc{TopoRec}s: \texttt{S2GAE} \cite{S2GAE}, \texttt{MaskGAE}-edge, \texttt{MaskGAE}-path~\cite{MaskGAE} (with edge masking and path masking separately). 
We use ``$\dagger$'' to mark baselines that are implemented by us for the current task because they are not officially implemented.

\begin{figure}
    \centering
    \includegraphics[width=0.48\textwidth]{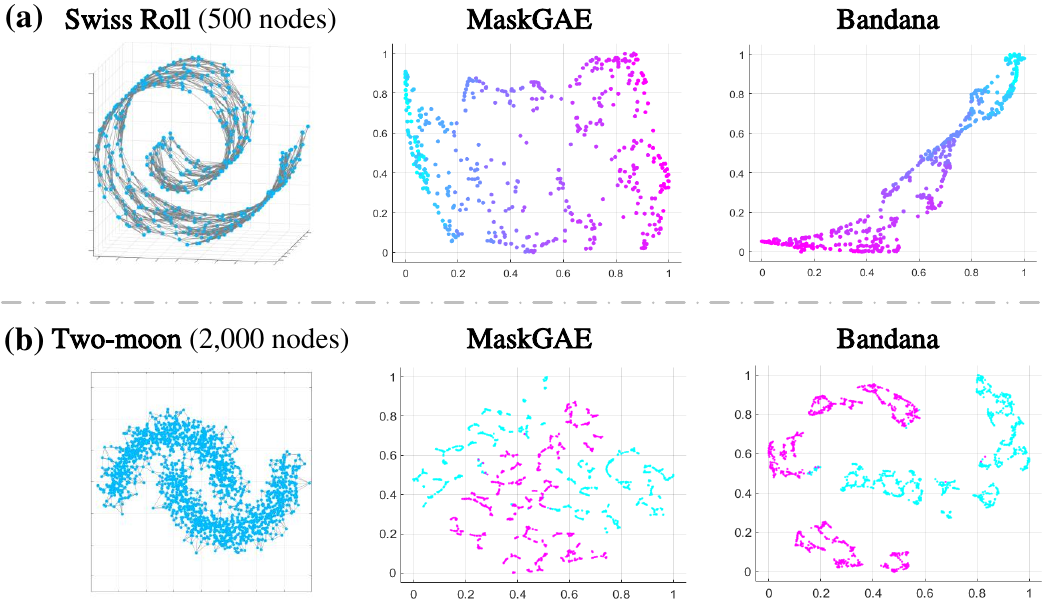}
    \caption{
    Manifold learning visualizations. \textmd{{\bf (a)} 
    %A surface in $\mathbb{R}^3$ can be obtained by expanding the spiral structure of 
    The Swiss Roll. \texttt{MaskGAE} only learns suboptimal representations loosely scattered in the latent space, whereas \texttt{Bandana} learns a more compact surface. {\bf (b)} The Two-moon.
    %containing two interleaved, crescent-shaped high-density regions. 
    While \texttt{MaskGAE} does not give informative results, \texttt{Bandana} successfully learns the crescent-shaped topology.}
    }
    \label{fig:manifold}
\end{figure}

\begin{table*}
\caption{\small 
Micro-F1(\%) and Macro-F1(\%) of node classification. \textmd{Best results in each column are in {\bf bold}. ``Avg. Rank'' stands for the average rank. ``OOM'' stands for ``Out-Of-Memory'' on a 24GB GPU.}
}
\label{table:nodeclas}
\small
\begin{center}
\adjustbox{max width=1.0\textwidth}{
\begin{tabular}{ccp{0.9in}ccccccccc}
\toprule[1.2pt]
Micro-F1 & \multirow{2}{*}{Year} & \multirow{2}{*}{
Model} & \multirow{2}{*}{Cora} & \multirow{2}{*}{CiteSeer} & \multirow{2}{*}{PubMed}& \multirow{2}{*}{Photo} & \multirow{2}{*}{Computers} & \multirow{2}{*}{CS} & \multirow{2}{*}{Physics} & \multirow{2}{*}{ogbn-arxiv} & \multirow{2}{*}{Avg. Rank} \\[-0.2em]
Macro-F1 & & & & & & & & \\[-0.1em]
\midrule
\multirow{4}{*}{\makecell{Traditional\\Autoencoder}} & \multirow{2}{*}{2016} & \multirow{2}{*}   {\texttt{GAE}$^\dagger$~\cite{GAE/VGAE}}& 80.15 ± 0.34 & 69.79 ± 0.36 & 80.51 ± 0.53 & 91.07 ± 0.09 & 87.92 ± 0.12 & 90.46 ± 0.29 & 93.04 ± 0.03 & 69.58 ± 0.32 & \multirow{2}{*}{10.3} \\ [-0.2em]
&& & 78.44 ± 0.90 & 62.64 ± 0.58 & 79.62 ± 0.38 & 89.83 ± 0.13 & 86.02 ± 0.25 & 89.75 ± 0.32 & 91.02 ± 0.04 & 48.25 ± 0.53 & \\ [-0.2em]
\arrayrulecolor[RGB]{180, 180, 180}\cmidrule(){2-12}\arrayrulecolor[RGB]{0, 0, 0}  
& \multirow{2}{*}{2018} &\multirow{2}{*}{\texttt{ARGA}$^\dagger$~\cite{ARGA/ARVGA}} & 77.93 ± 0.59 & 68.55 ± 0.34 & 77.78 ± 0.63 & 92.77 ± 0.26 & 88.11 ± 0.08 & 92.46 ± 0.14 & 94.32 ± 0.04 & 69.81 ± 0.27 & \multirow{2}{*}{9.3} \\[-0.2em]
&& & 76.89 ± 0.51 & 63.33 ± 0.68 & 76.54 ± 0.82 & 91.60 ± 0.10 & 86.34 ± 0.15 & 90.61 ± 0.19 & 92.58 ± 0.07 & 47.89 ± 0.45 & 
\\[-0.2em] \cmidrule(){1-12} 
\multirow{6}{*}{\makecell{Variational\\Autoencoder}} & \multirow{2}{*}{2016} &\multirow{2}{*}{\texttt{VGAE}$^\dagger$~\cite{GAE/VGAE}} & 76.30 ± 0.49 & 58.85 ± 0.79 & 75.73 ± 0.22 & 89.58 ± 0.20 & 84.99 ± 0.19 & 92.33 ± 0.07 & 94.40 ± 0.07 & 69.94 ± 0.30 & \multirow{2}{*}{12.3} \\[-0.2em]
&& & 74.70 ± 0.60 & 52.80 ± 0.91 & 75.39 ± 0.24 & 86.61 ± 0.40 & 82.26 ± 0.30 & 89.09 ± 0.16 & 92.28 ± 0.11 & 47.05 ± 0.79 &
\\[-0.2em] \arrayrulecolor[RGB]{180, 180, 180}\cmidrule(){2-12}\arrayrulecolor[RGB]{0, 0, 0}
& \multirow{2}{*}{2018} &\multirow{2}{*}{\texttt{ARVGA}$^\dagger$~\cite{ARGA/ARVGA}} & 76.85 ± 0.88 & 54.73 ± 0.46 & 73.06 ± 0.42 & 89.51 ± 0.23 & 85.03 ± 0.15 & 92.56 ± 0.09 & 93.64 ± 0.08 & 69.39 ± 0.36  & \multirow{2}{*}{12.5} \\[-0.2em]
&& & 75.15 ± 0.95 & 48.71 ± 1.27 & 73.49 ± 0.44 & 86.88 ± 0.40 & 81.54 ± 0.36 & 89.67 ± 0.23 & 91.08 ± 0.15 & 47.34 ± 0.59  &
\\[-0.2em] \arrayrulecolor[RGB]{180, 180, 180}\cmidrule(){2-12}\arrayrulecolor[RGB]{0, 0, 0} 
& \multirow{2}{*}{2023} &\multirow{2}{*}{\texttt{SeeGera}~\cite{SeeGera}} & 83.95 ± 0.55 & 72.11 ± 1.26 & 79.55 ± 0.29 & 90.13 ± 0.57 & 88.39 ± 0.26* & 88.79 ± 0.93 & \multirow{2}{*}{OOM} & \multirow{2}{*}{OOM} & \multirow{2}{*}{9.6} \\[-0.2em]
&& & 82.88 ± 0.66 & 68.48 ± 0.86 & 78.36 ± 0.44 & 87.76 ± 1.01 & & 85.38 ± 1.62 & & &
\\[-0.2em] \cmidrule(){1-12} 
\multirow{8}{*}{\makecell{Contrastive \&\\Non-contrastive}} &
\multirow{2}{*}{2020} &\multirow{2}{*}{\texttt{GRACE}~\cite{GRACE}} & 80.95 ± 0.29 & 70.39 ± 0.46 & 83.55 ± 0.44 & 92.12 ± 0.14 & 87.68 ± 0.15 & 91.90 ± 0.01 & \multirow{2}{*}{OOM} & \multirow{2}{*}{OOM} & \multirow{2}{*}{8.3} \\[-0.2em]
&& & 79.20 ± 0.44 & 68.15 ± 0.32 & 83.29 ± 0.20 & 90.99 ± 0.36 & 85.82 ± 0.27 & 89.09 ± 0.01 & & &
\\[-0.2em] \arrayrulecolor[RGB]{180, 180, 180}\cmidrule(){2-12}\arrayrulecolor{} 
& \multirow{2}{*}{2021} & \multirow{2}{*}{\texttt{GCA}~\cite{GCA}} & 81.92 ± 0.17 & 71.60 ± 0.27 & {\bf 84.08 ± 0.16} & 92.39 ± 0.20 & 87.14 ± 0.15 & 92.61 ± 0.06 & \multirow{2}{*}{OOM} & \multirow{2}{*}{OOM} & \multirow{2}{*}{6.4} \\[-0.2em]
&& & 80.76 ± 0.35 & {\bf 68.79 ± 0.37} & {\bf 83.70 ± 0.29} & 91.17 ± 0.30 & 85.10 ± 0.31 & 90.64 ± 0.16 & & &
\\[-0.2em] \arrayrulecolor[RGB]{180, 180, 180}\cmidrule(){2-12}\arrayrulecolor[RGB]{0, 0, 0} 
& \multirow{2}{*}{2022} &\multirow{2}{*}{\texttt{COSTA}~\cite{COSTA}} & 84.60 ± 0.20 & 72.57 ± 0.31 & 83.76 ± 0.03 & 90.98 ± 0.00 & 87.35 ± 0.08 & 92.48 ± 0.05 & 95.31 ± 0.04 & \multirow{2}{*}{OOM} & \multirow{2}{*}{6.4} \\[-0.2em]
&& & 82.50 ± 0.21 & 66.11 ± 0.29 & 83.16 ± 0.02 & 88.22 ± 0.00 & 85.99 ± 0.13 & 89.32 ± 0.08 & 93.90 ± 0.05 & &
\\[-0.2em] \arrayrulecolor[RGB]{180, 180, 180}\cmidrule(){2-12}\arrayrulecolor[RGB]{0, 0, 0} 
& \multirow{2}{*}{2021} &\multirow{2}{*}{\texttt{CCA-SSG}~\cite{CCA-SSG}} & 83.96 ± 0.38 & 73.45 ± 0.44 & 81.81 ± 0.53 & 92.87 ± 0.36 & 88.61 ± 0.29 & 93.01 ± 0.29 & 95.31 ± 0.07 & 69.52 ± 0.09 & \multirow{2}{*}{5.1} \\[-0.2em]
&& & 83.01 ± 0.48 & 68.75 ± 0.51 & 81.31 ± 0.51 & 91.69 ± 0.49 & 87.24 ± 0.52 & 90.73 ± 0.51 & 93.76 ± 0.10 & 47.39 ± 0.51 &
\\[-0.2em] \cmidrule(){1-12} 
\multirow{4}{*}{\textsc{FeatRec}} &
\multirow{2}{*}{2022} &\multirow{2}{*}{\texttt{GraphMAE}~\cite{GraphMAE}} & 84.05 ± 0.59 & 73.06 ± 0.37 & 80.98 ± 0.47 & 92.92 ± 0.40 & 89.24 ± 0.45 & 93.09 ± 0.14 & {\bf 95.65 ± 0.07} & 71.30 ± 0.24 & \multirow{2}{*}{3.6} \\[-0.2em]
&& & {\bf 83.07 ± 0.53} & 67.78 ± 0.85 & 80.26 ± 0.48 & 91.93 ± 0.47 & {\bf 88.12 ± 0.78} & {\bf 91.42 ± 0.15} & 94.19 ± 0.09 & {\bf 51.15 ± 0.15} &
\\[-0.2em] \arrayrulecolor[RGB]{180, 180, 180}\cmidrule(){2-12}\arrayrulecolor[RGB]{0, 0, 0} 
& \multirow{2}{*}{2023} &\multirow{2}{*}{\texttt{GraphMAE2}~\cite{GraphMAE2}} & 83.84 ± 0.54 & 73.48 ± 0.34 & 81.34 ± 0.44 & 93.30 ± 0.20 & 89.01 ± 1.53 & 91.31 ± 0.07 & 95.25 ± 0.05 & {\bf 71.82 ± 0.00} & \multirow{2}{*}{5.2} \\[-0.2em]
&& & 82.80 ± 0.46 & 68.70 ± 0.42 & 80.68 ± 0.43 & 92.19 ± 0.24 & 87.63 ± 1.79 & 88.89 ± 0.13 & 93.78 ± 0.07 & 50.42 ± 0.00 &
\\[-0.2em] \cmidrule(){1-12} 
\multirow{8}{*}{\textsc{TopoRec}} &
\multirow{2}{*}{2023} &\multirow{2}{*}{\texttt{S2GAE}~\cite{S2GAE}} & 78.34 ± 0.96 & 65.31 ± 0.64 & 80.11 ± 0.52 & 91.43 ± 0.07 & 85.31 ± 0.07 & 90.47 ± 0.07 & 93.98 ± 0.06 & 67.77 ± 0.36 & \multirow{2}{*}{11.8} \\[-0.2em]
&& & 77.44 ± 0.86 & 62.54 ± 0.64 & 79.04 ± 0.47 & 90.47 ± 0.15 & 81.48 ± 0.18 & 87.69 ± 0.11 & 91.95 ± 0.08 & 36.41 ± 0.24 &
\\[-0.2em] \arrayrulecolor[RGB]{180, 180, 180}\cmidrule(){2-12}\arrayrulecolor[RGB]{0, 0, 0}  
& \multirow{2}{*}{2023} &\multirow{2}{*}{\texttt{MaskGAE}-edge~\cite{MaskGAE}} & 83.33 ± 0.15 & 72.02 ± 0.46 & 82.33 ± 0.39 & 93.28 ± 0.08 & 89.42 ± 0.15 & 92.29 ± 0.25 & 95.10 ± 0.04 & 70.95 ± 0.29 & \multirow{2}{*}{5.9} \\[-0.2em]
&& & 82.60 ± 0.24 & 66.36 ± 0.63 & 81.63 ± 0.41 & 92.04 ± 0.08 & 88.00 ± 0.14 & 90.17 ± 0.34 & 93.48 ± 0.04 & 49.37 ± 0.45 &
\\[-0.2em] \arrayrulecolor[RGB]{180, 180, 180}\cmidrule(){2-12}\arrayrulecolor[RGB]{0, 0, 0}  
& \multirow{2}{*}{2023} &\multirow{2}{*}{\texttt{MaskGAE}-path~\cite{MaskGAE}} & 82.54 ± 0.16 & 72.32 ± 0.39 & 82.80 ± 0.22 & 93.29 ± 0.10 & 89.40 ± 0.10 & 92.54 ± 0.21 & 95.15 ± 0.11 & 71.22 ± 0.40 & \multirow{2}{*}{5.4} \\[-0.2em]
&& & 81.84 ± 0.26 & 65.77 ± 0.40 & 82.23 ± 0.23 & 92.16 ± 0.17 & 87.69 ± 0.15 & 90.25 ± 0.31 & 93.51 ± 0.03 & 49.99 ± 0.54 &
\\[-0.2em] \arrayrulecolor[RGB]{180, 180, 180}\cmidrule(){2-12}\arrayrulecolor[RGB]{0, 0, 0} 
\rowcolor{gray!25}
\cellcolor{white} && & {\bf 84.62 ± 0.37} & {\bf 73.60 ± 0.16} & 83.53 ± 0.51 & {\bf 93.44 ± 0.11} & {\bf 89.62 ± 0.09} & {\bf 93.10 ± 0.05} & 95.57 ± 0.04 & 71.09 ± 0.24 & \\[-0.2em]
\rowcolor{gray!25}
\cellcolor{white} & \multirow{-2}{*}{2024} &\multirow{-2}{*}{\texttt{Bandana}} & 82.97 ± 0.92 & 68.11 ± 0.48 & 82.99 ± 0.40 & {\bf 92.26 ± 0.04} & 87.79 ± 0.20 & 91.02 ± 0.13 & {\bf 94.20 ± 0.05} & 49.66 ± 0.50 &  \multirow{-2}{*}{\bf 2.4}\\[-0.2em]
\bottomrule[1.2pt]
\multicolumn{12}{l}{\small *We obtain a much lower score for \texttt{SeeGera} on Computers than the official one. We report the Micro-F1 from the original paper~\cite{SeeGera} instead.} 
\end{tabular}}
\end{center}
\end{table*}

\subsection{Learning Topological Manifolds (RQ1)}\label{5.2}

We verify \texttt{Bandana}'s informative representation learning ability by performing manifold learning on two undirected synthetic datasets with structures: {\it Swiss Roll}, a curved surface on $\mathbb{R}^3$; and {\it Two-moon}, two interleaved crescent-shaped clusters on $\mathbb{R}^2$. We assign a column of the identity matrix $\mathbf{I}_n$ to each node as features with no topological information. 
The latent representation spaces are learned by \texttt{MaskGAE} and \texttt{Bandana} respectively (each model is trained till the early stopping) and visualized by t-SNE~\cite{t-SNE}.
It is obvious in Figure \ref{fig:manifold} that \texttt{Bandana} is more topologically informative than \texttt{MaskGAE}.

\setlength{\fboxsep}{0.5pt}
\begin{table*}
\setlength{\tabcolsep}{6pt}
\caption{\small 
AUC(\%) and AP(\%) of link prediction. \textmd{Best results in each column are in {\bf bold}. ``OOM'' stands for ``Out-Of-Memory'' on a 24GB GPU.}
}
\label{table:linkpred}
\small
\begin{center}
\adjustbox{max width=1.0\textwidth}{
\begin{tabular}{ccp{0.9in}cccccccc}
\toprule[1.2pt]
\makecell{AUC\\AP} 
& Year & Model & Cora & CiteSeer & PubMed & Photo & Computers & CS & Physics 
& Avg. Rank \\[-0.2em]
\midrule
\multirow{4}{*}{\makecell{Traditional\\Autoencoder}} & \multirow{2}{*}{2016} & \multirow{2}{*}{\texttt{GAE}~\cite{GAE/VGAE}} & 94.66 ± 0.26 & 95.19 ± 0.45 & 94.58 ± 1.12 & 71.45 ± 0.95 & 70.99 ± 1.03 & 93.78 ± 0.36 & 88.88 ± 1.11 & 
\multirow{2}{*}{9.1} 
\\ [-0.2em]
&& & 94.22 ± 0.39 & 95.70 ± 0.31 & 94.26 ± 1.65 & 65.99 ± 0.96 & 67.88 ± 0.82 & 89.87 ± 0.59 & 82.45 ± 1.59
\\[-0.2em]
\arrayrulecolor[RGB]{180, 180, 180}\cmidrule(){2-11}\arrayrulecolor[RGB]{0, 0, 0}  
& \multirow{2}{*}{2018} & \multirow{2}{*}{\texttt{ARGA}~\cite{ARGA/ARVGA}} & 94.76 ± 0.18 & 95.68 ± 0.35 & 94.12 ± 0.08 & 85.42 ± 0.79 & 67.09 ± 3.93 & 95.49 ± 0.17 & 90.70 ± 1.08 & 
\multirow{2}{*}{7.0}
\\[-0.2em]
&& & 94.93 ± 0.20 & 96.34 ± 0.25 & 94.19 ± 0.08 & 80.58 ± 1.40 & 62.53 ± 3.17 & 92.56 ± 0.33 & 89.37 ± 1.16
\\[-0.2em] \cmidrule(){1-11} 
\multirow{8}{*}{\makecell{Variational\\Autoencoder}} & \multirow{2}{*}{2016} & \multirow{2}{*}{\texttt{VGAE}~\cite{GAE/VGAE}} & 91.24 ± 0.48 & 94.55 ± 0.48 & 95.46 ± 0.04 & 95.61 ± 0.05 & 92.69 ± 0.03 & 87.34 ± 0.43 & 89.27 ± 0.83 & 
\multirow{2}{*}{6.7} 
\\[-0.2em]
&& & 92.27 ± 0.43 & 95.34 ± 0.37 & 94.29 ± 0.07 & 94.63 ± 0.06 & 88.27 ± 0.08 & 80.24 ± 0.55 & 82.79 ± 1.14 
\\[-0.2em] \arrayrulecolor[RGB]{180, 180, 180}\cmidrule(){2-11}\arrayrulecolor[RGB]{0, 0, 0}
& \multirow{2}{*}{2018} & \multirow{2}{*}{\texttt{ARVGA}~\cite{ARGA/ARVGA}} & 91.35 ± 0.87 & 94.47 ± 0.33 & 96.17 ± 0.21 & 95.44 ± 0.14 & 92.38 ± 0.15 & 87.39 ± 0.37 & 88.96 ± 0.96 & 
\multirow{2}{*}{7.0}
\\[-0.2em]
&& & 91.98 ± 0.85 & 95.21 ± 0.33 & 94.81 ± 0.41 & 94.49 ± 0.12 & 88.49 ± 0.33 & 80.31 ± 0.49 & 82.38 ± 1.31
\\[-0.2em] \arrayrulecolor[RGB]{180, 180, 180}\cmidrule(){2-11}\arrayrulecolor[RGB]{0, 0, 0} 
& \multirow{2}{*}{2019} & \multirow{2}{*}{\texttt{SIG-VAE}~\cite{SIG-VAE}} & 90.36 ± 1.34 & 88.85 ± 0.69 & \multirow{2}{*}{OOM} & \multirow{2}{*}{OOM} & \multirow{2}{*}{OOM} & \multirow{2}{*}{OOM} & \multirow{2}{*}{OOM} & 
\multirow{2}{*}{11.0}
\\[-0.2em]
&& & 91.36 ± 1.16 & 90.27 ± 0.73 & & & & &
\\[-0.2em] \arrayrulecolor[RGB]{180, 180, 180}\cmidrule(){2-11}\arrayrulecolor[RGB]{0, 0, 0} 
& \multirow{2}{*}{2023} & \multirow{2}{*}{\texttt{SeeGera}~\cite{SeeGera}} & 95.49 ± 0.70 & 94.61 ± 1.05 & 95.19 ± 3.94 & 95.25 ± 1.19 & 96.53 ± 0.16 & 95.73 ± 0.70 & \multirow{2}{*}{OOM} & 
\multirow{2}{*}{3.8}
\\[-0.2em]
&& & {\bf 95.90 ± 0.64} & 96.40 ± 0.89 & 94.60 ± 4.17 & 94.04 ± 1.18 & 96.33 ± 0.16 & 93.17 ± 0.53
\\[-0.2em] \cmidrule(){1-11} 
\multirow{8}{*}{\makecell{Contrastive \&\\Non-contrastive}} &
\multirow{2}{*}{2020} & \multirow{2}{*}{\texttt{GRACE}$^\dagger$~\cite{GRACE}} & 81.80 ± 0.45 & 84.78 ± 0.38 & 93.11 ± 0.37 & 88.64 ± 1.17 & 89.97 ± 0.25 & 87.67 ± 0.10 & \multirow{2}{*}{OOM} & 
\multirow{2}{*}{9.2}
\\[-0.2em]
&& & 82.02 ± 0.50 & 82.85 ± 0.36 & 92.88 ± 0.30 & 83.85 ± 4.15 & 92.15 ± 0.43 & 94.87 ± 0.02 &
\\[-0.2em] \arrayrulecolor[RGB]{180, 180, 180}\cmidrule(){2-11}\arrayrulecolor{} 
& \multirow{2}{*}{2021} & \multirow{2}{*}{\texttt{GCA}$^\dagger$~\cite{GCA}} & 81.91 ± 0.76 & 84.72 ± 0.28 & 94.33 ± 0.67 & 89.61 ± 1.46 & 90.67 ± 0.30 & 88.05 ± 0.00 & \multirow{2}{*}{OOM} & 
\multirow{2}{*}{8.7}
\\[-0.2em]
&& & 80.51 ± 0.71 & 81.57 ± 0.22 & 93.13 ± 0.62 & 86.53 ± 3.00 & 90.50 ± 0.63 & 94.94 ± 0.37
\\[-0.2em] \arrayrulecolor[RGB]{180, 180, 180}\cmidrule(){2-11}\arrayrulecolor[RGB]{0, 0, 0} 
& \multirow{2}{*}{2021} & \multirow{2}{*}{\texttt{CCA-SSG}$^\dagger$~\cite{CCA-SSG}} & 67.54 ± 1.30 & 78.88 ± 2.73 & 74.97 ± 0.28 & 91.04 ± 2.98 & 83.85 ± 1.35 & 83.54 ± 0.98 & 77.40 ± 0.08 & 
\multirow{2}{*}{12.4}
\\[-0.2em]
&& & 72.74 ± 1.18 & 77.42 ± 4.56 & 77.11 ± 0.26 & 89.68 ± 3.85 & 84.04 ± 1.74 & 78.66 ± 1.06 & 73.33 ± 0.10
\\[-0.2em] \arrayrulecolor[RGB]{180, 180, 180}\cmidrule(){2-11}\arrayrulecolor[RGB]{0, 0, 0} 
& \multirow{2}{*}{2023} & \multirow{2}{*}{\texttt{T-BGRL}~\cite{T-BGRL}} & 73.18 ± 0.54 & 78.11 ± 0.48 & 76.21 ± 0.18 & 80.80 ± 0.04 & 84.60 ± 0.05 & 70.08 ± 0.12 & 89.18 ± 0.04 &
\multirow{2}{*}{11.4} 
\\[-0.2em]
&& & 76.81 ± 0.73 & 83.15 ± 0.47 & 80.99 ± 0.13 & 84.34 ± 0.06 & 86.85 ± 0.05 & 79.50 ± 0.09 & 84.30 ± 0.05
\\[-0.2em] \cmidrule(){1-11} 
\multirow{4}{*}{\textsc{FeatRec}} &
\multirow{2}{*}{2022} & \multirow{2}{*}{\texttt{GraphMAE}$^\dagger$~\cite{GraphMAE}} & 93.02 ± 0.53 & 95.21 ± 0.47 & 87.54 ± 1.06 & 75.08 ± 1.24 & 71.27 ± 0.89 & 92.45 ± 4.18 & 85.03 ± 7.16 & \multirow{2}{*}{10.3} 
\\[-0.2em]
&& & 91.40 ± 0.59 & 94.42 ± 0.67 & 86.93 ± 1.01 & 70.04 ± 1.12 & 66.84 ± 1.10 & 91.67 ± 4.17 & 82.46 ± 9.33 &
\\[-0.2em] \arrayrulecolor[RGB]{180, 180, 180}\cmidrule(){2-11}\arrayrulecolor[RGB]{0, 0, 0} 
& \multirow{2}{*}{2023} & \multirow{2}{*}{\texttt{GraphMAE2}$^\dagger$~\cite{GraphMAE2}} & 93.26 ± 1.00 & 95.26 ± 0.14 & 90.85 ± 0.91 & 73.03 ± 2.24 & 72.20 ± 2.09 & 94.57 ± 0.32 & 94.56 ± 0.81 & 
\multirow{2}{*}{8.4}
\\[-0.2em]
&& & 91.65 ± 0.98 & 94.36 ± 0.20 & 90.37 ± 0.92 & 68.77 ± 1.50 & 67.97 ± 1.52 & 92.76 ± 0.54 & 93.86 ± 1.09 &
\\[-0.2em] \cmidrule(){1-11} 
\multirow{8}{*}{\textsc{TopoRec}} &
\multirow{2}{*}{2023} & \multirow{2}{*}{\texttt{S2GAE}~\cite{S2GAE}} & 89.27 ± 0.33 & 86.35 ± 0.42 & 89.53 ± 0.23 & 86.80 ± 2.85 & 84.16 ± 4.82 & 86.60 ± 1.06 & 88.92 ± 1.24 & 
\multirow{2}{*}{10.1}
\\[-0.2em]
&& & 89.78 ± 0.22 & 87.38 ± 0.29 & 88.68 ± 0.33 & 80.56 ± 3.74 & 78.13 ± 6.58 & 82.93 ± 1.63 & 88.20 ± 1.34 &
\\[-0.2em] \arrayrulecolor[RGB]{180, 180, 180}\cmidrule(){2-11}\arrayrulecolor[RGB]{0, 0, 0}  
& \multirow{2}{*}{2023} & \multirow{2}{*}{\texttt{MaskGAE}-edge~\cite{MaskGAE}} & 95.66 ± 0.16 & 97.02 ± 0.27 & 96.51 ± 0.82 & 81.12 ± 0.45 & 76.23 ± 3.13 & 92.41 ± 0.44 & 91.94 ± 0.37 & \multirow{2}{*}{5.9}
\\[-0.2em]
&& & 94.65 ± 0.24 & 96.89 ± 0.45 & 96.08 ± 0.68 & 77.11 ± 0.40 & 71.71 ± 2.90 & 87.16 ± 0.69 & 86.33 ± 0.55
\\[-0.2em] \arrayrulecolor[RGB]{180, 180, 180}\cmidrule(){2-11}\arrayrulecolor[RGB]{0, 0, 0}  
& \multirow{2}{*}{2023} & \multirow{2}{*}{\texttt{MaskGAE}-path~\cite{MaskGAE}} & 95.47 ± 0.25 & {\bf 97.21 ± 0.17} & 97.19 ± 0.18 & 80.46 ± 0.34 & 73.24 ± 1.26 & 87.96 ± 0.44 & 86.19 ± 0.36 & 
\multirow{2}{*}{7.4} 
\\[-0.2em]
&& & 94.64 ± 0.25 & 97.02 ± 0.32 & 96.69 ± 0.19 & 76.56 ± 0.55 & 70.94 ± 1.26 & 80.84 ± 0.58 & 78.55 ± 0.45
\\[-0.2em] \arrayrulecolor[RGB]{180, 180, 180}\cmidrule(){2-11}\arrayrulecolor[RGB]{0, 0, 0} 
\rowcolor{gray!25}
\cellcolor{white} && & {\bf 95.71 ± 0.12} & 96.89 ± 0.21 & {\bf 97.26 ± 0.16} & {\bf 97.24 ± 0.11} & {\bf 97.33 ± 0.06} & {\bf 97.42 ± 0.08} & {\bf 97.02 ± 0.04} &
\\[-0.2em]
\rowcolor{gray!25}
\cellcolor{white} & \multirow{-2}{*}{2024} & \multirow{-2}{*}{\texttt{Bandana}} & 95.25 ± 0.16 & {\bf 97.16 ± 0.17} & {\bf 96.74 ± 0.38} & {\bf 96.79 ± 0.15} & {\bf 96.91 ± 0.09} & {\bf 97.09 ± 0.15} & {\bf 96.67 ± 0.05} & 
\multirow{-2}{*}{\bf 1.2}
\\[-0.2em]
\bottomrule[1.2pt]
\end{tabular}}
\end{center}
\end{table*}

\subsection{Comparison on Node Classification (RQ2)}\label{5.3}

Similar to other self-supervised models~\cite{DGI, CCA-SSG, MAE, GraphMAE, MaskGAE}, \texttt{Bandana} follows the {\bf linear probing} setup to evaluate. That is, we use the pre-trained encoder's output representations to train a Xavier-initialized~\cite{Xavier} linear layer.

We report two classic metrics, Micro-F1 and Macro-F1, in Table \ref{table:nodeclas}. It is evident that \texttt{Bandana} achieves competitive performance with state-of-the-art contrastive methods and \textsc{FeatRec}s, but only with a structure-learning pretext. As indicated by the Avg. Rank, the performance of discrete \textsc{TopoRec}s (\texttt{S2GAE}, \texttt{MaskGAE}) in node classification tasks is difficult to emulate the dominant \textsc{FeatRec}s (\texttt{GraphMAE}, \texttt{GraphMAE2}). However, 
\texttt{Bandana} surpasses both settings of \texttt{MaskGAE} on 7/8 datasets, surpasses \texttt{COSTA} (one of the most advanced contrastive frameworks) on 6/7 datasets, and outperforms \texttt{GraphMAE} and \texttt{GraphMAE2} by 1.2 and 2.8 ranks, respectively.
Our work, perhaps surprisingly, shows that fine-grained topological learning can uncover the close relationship between the graph structure and the intrinsic characteristics of node features. 

\begin{table}
\setlength{\tabcolsep}{2pt}
\caption{\small 
Effect of the masking strategies on the average node classification accuracy (\%). \textmd{Best results in each column are in {\bf bold}. Second-best results in each column are \underline{underlined}.}
}
\label{table:maskablation}
\begin{center}
\adjustbox{max width=0.48\textwidth}{
\begin{tabular}{p{1.8in}ccc}
\toprule[1.2pt]
Variants & Cora & CiteSeer & PubMed \\
\midrule
Bernoulli & 79.16 ± 0.15 & 68.60 ± 0.90 & 82.67 ± 0.40 \\
Uniform & 81.36 ± 0.20 & 70.25 ± 0.55 & 81.84 ± 0.47 \\
Truncated Gaussian & 79.34 ± 0.46 & 69.95 ± 0.25 & 82.00 ± 0.56 \\
Boltzmann-Gibbs & \underline{84.02 ± 0.09} & \underline{72.45 ± 0.42} & \underline{83.31 ± 0.38} \\
Boltzmann-Gibbs, LWM & 82.38 ± 0.19 & 70.75 ± 0.55 & 81.70 ± 0.57 \\
\rowcolor{gray!25}
\texttt{Bandana} (Boltzmann-Gibbs, LWP) & {\bf 84.62 ± 0.37} & {\bf 73.60 ± 0.16} & {\bf 83.53 ± 0.51} \\
\bottomrule[1.2pt]
\end{tabular}}
\end{center}
\end{table}

\subsection{Comparison on Link Prediction (RQ3)}\label{5.4}
Unlike node classification, our evaluation of link prediction is {\it different from the old routine}. Previous \textsc{TopoRec}s directly performed link prediction in an end-to-end manner without probing or fine-tuning since they do the exact same thing for pre-training. However, it is not a self-supervised case and hence not suitable for evaluating self-supervised models. Thus, we utilize a fairer evaluation scheme called {\bf dot-product probing}, which replaces the original MLP decoder with a dot-product operator $\mathbf{A}_\text{recon}=\sigma(\mathbf{Z}\mathbf{Z}^\top)$, as \texttt{SeeGera} does~\cite{SeeGera}. We 
%switch the evaluation scheme to 
employ the dot-product probing instead of the end-to-end training for \texttt{Bandana} {\it as well as all baselines} 
(note that this may lead to some discordance between our results and those officially reported). 
Further analyses of the dot-product probing can be found in Appendix \ref{E.2}.

We report Area Under the ROC curve (AUC) and Average Precision (AP) in Table \ref{table:linkpred}. 
We have several observations. (\lowercase{\romannumeral 1}) Despite no longer using link prediction for pre-training, \texttt{Bandana} still achieves the best link prediction results. In particular, it
greatly outperforms the performance of \texttt{MaskGAE} by 20\% on Computers.
(\lowercase{\romannumeral 2}) \texttt{Bandana} gains over 3\%-10\% improvement compared to the best contrastive results. From the Avg. Rank, the performance of contrastive methods under dot-product probing is less than satisfactory, even for the advanced link prediction model \texttt{T-BGRL}, because they do not explicitly learn graph structures while pre-training.
(\lowercase{\romannumeral 3}) \textsc{FeatRec}s (\texttt{GraphMAE}, \texttt{GraphMAE2}) do not perform as well as \textsc{TopoRec}s and even traditional autoencoders (\texttt{GAE}, \texttt{ARGA}), since they only pay attention to node features. (\lowercase{\romannumeral 4}) Some contrastive methods and variational autoencoders 
require more memory for large graphs. This highlights the lightweight property of \texttt{TopoRec}s. 

\subsection{Ablation Study of Masking Strategies (RQ4)}\label{5.5}

We have analyzed the strengths of Boltzmann-Gibbs bandwidths and the layer-wise strategy in Section \ref{4.1}. To validate these strengths, we experiment with different distributions, including the discrete Bernoulli distribution $M_{ij} \sim Bernoulli(1-p)$, uniform distribution $M_{ij} \sim U(0, 2-2p) (p>0.5)$, truncated Gaussian distribution $\scriptsize M_{ij} \sim \psi(1-p,1,0,2-2p) (p>0.5)$\footnote{$\psi(\mu,\sigma^2,a,b)$ denotes a Gaussian distribution $\mathcal{N}(\mu,\sigma^2)$ truncated within the interval $[a,b]$ where $-\infty<a<b<+\infty$.}, and the Boltzmann-Gibbs distribution in \cref{eq:bandwidth} 
(we ensure that masks sampled from these distributions have the same mask ratio $p$). 
These variants only feed the output-layer representations into the decoder. The model employing layer-wise masking only (each layer uses an independent mask set but only the last layer performs the prediction) is referred to as LWM, while the one with both layer-wise masking and prediction is referred to as LWP. 
It is obvious from Table \ref{table:maskablation} that the model with Boltzmann-Gibbs bandwidths outperforms all models with different mask distributions. Furthermore, \texttt{Bandana}'s setting obtains the best node classification performance on all three datasets. Note that the model with LWM only learns suboptimal representations because it only attempts to predict one set of masks while multiple different sets are used.

\section{Conclusion}

This work firstly discusses two limitations in the message propagation of existing discrete \textsc{TopoRec}s, which induce the insufficiency of learning topologically informative representations. To address the issues, we explore non-discrete masking by a novel bandwidth masking and reconstruction scheme. We present our masked graph autoencoder \texttt{Bandana} via the specialized Boltzmann-Gibbs masking and layer-wise prediction, and thoroughly explore its empirical and theoretical superiority. We demonstrate that \texttt{Bandana} can learn more precise graph manifolds and outperform other baselines, including the state-of-the-art contrastive methods and \textsc{FeatRec}s, on link prediction and the feature-related node classification, solely by pre-training on a structure-learning pretext. 
While \texttt{Bandana} may not represent the optimal solution, it is the first attempt to explore a new paradigm for masked graph autoencoders that diverges from the discrete mask-then-reconstruct stereotype.

%
% The acknowledgments section is defined using the "acks" environment
% (and NOT an unnumbered section). This ensures the proper
% identification of the section in the article metadata, and the
% consistent spelling of the heading.
\begin{acks}
This work is supported by National Natural Science Foundation of China under grants 62376103, 62302184, 62206102, U1936108 and Science and Technology Support Program of Hubei Province under grant 2022BAA046.
\end{acks}

%%
%% The next two lines define the bibliography style to be used, and
%% the bibliography file.
\bibliographystyle{ACM-Reference-Format}
\bibliography{sample-base}

\newpage
% \clearpage
%%
%% If your work has an appendix, this is the place to put it.
\appendix

\section{More Theoretical Details}

This section gives the detailed proofs of propositions and theorems.

\subsection{Proof of Theorem \ref{thm:dirichlet}}\label{A.1}
\begin{proof}
It is obvious that the unmasked ego-graph $\mathcal{G}_{i,m}$ has $p(n_i - 1) + 1$ nodes. So
\begin{align}
& E_D(\mathcal{G}_{i}) - E_D(\mathcal{G}_{i,m}) \notag\\
&= \frac{n_i-1}{n_i}\|\boldsymbol{X}_i^{\mathcal{G}_{i}}-\boldsymbol{X}_j^{\mathcal{G}_{i}}\|^2 - \frac{p(n_i-1)}{p(n_i-1)+1}\|\boldsymbol{X}_i^{\mathcal{G}_{i}}-\boldsymbol{X}_j^{\mathcal{G}_{i}}\|^2 \notag\\
&= \frac{(n_i-1)(p(n_i-1)+1)-pn_i(n_i-1)}{pn_i(n_i-1)+n_i}\|\boldsymbol{X}_i^{\mathcal{G}_{i}}-\boldsymbol{X}_j^{\mathcal{G}_{i}}\|^2 \notag\\
&= \frac{(n_i-1)(1-p)}{pn_i(n_i-1)+n_i}\|\boldsymbol{X}_i^{\mathcal{G}_{i}}-\boldsymbol{X}_j^{\mathcal{G}_{i}}\|^2 \notag\\
&\ge 0
\end{align}
and $E_D(\mathcal{G}_{i}) - E_D(\mathcal{G}_{i,m}) = 0$ iff $p=1$.
\end{proof}

\subsection{Proof of Proposition \ref{thm:dae}}\label{A.2}
\begin{proof}
Under the assumption of $\tilde{\boldsymbol{T}}_{j}\sim\mathcal{N}(\boldsymbol{\mu}_{\tilde{\boldsymbol{T}}_{j}},\mathbf{\Sigma}_{\tilde{\boldsymbol{T}}_{j}})$ and $\mathbf{\Sigma}_{\tilde{\boldsymbol{T}}_{j}}=c\mathbf{I}$, $r_{\mathbf{X}}(\tilde{\boldsymbol{T}}_{j})$ is a predictor of the Gaussian mean $\boldsymbol{\mu}_{\tilde{\boldsymbol{T}}_{j}}$. As such, the negative log likelihood in \cref{eq:bandce} can be rewritten as an $\ell_2$ error of topological encoding:
\begin{align}
\mathcal{L} &= -\mathbb{E}_{j\in\mathcal{V}}[\log r_{\mathbf{X}}(\tilde{\boldsymbol{T}}_{j})] \notag\\
&= -\mathbb{E}_{j\in\mathcal{V}}\left[\log \frac{\exp\left(-\frac{1}{2}(r_{\mathbf{X}}(\tilde{\boldsymbol{T}}_{j})-\tilde{\boldsymbol{T}}_{j})^\top\mathbf{\Sigma}_{\tilde{\boldsymbol{T}}_{j}}^{-1}(r_{\mathbf{X}}(\tilde{\boldsymbol{T}}_{j})-\tilde{\boldsymbol{T}}_{j})\right)  }{(2\pi)^\frac{n}{2}\det(\mathbf{\Sigma}_{\tilde{\boldsymbol{T}}_{j}})^\frac{1}{2}} 
\right] \notag\\
&=\frac{1}{2c}\mathbb{E}_{j\in\mathcal{V}}\left[(r_{\mathbf{X}}(\tilde{\boldsymbol{T}}_{j})-\tilde{\boldsymbol{T}}_{j})^\top(r_{\mathbf{X}}(\tilde{\boldsymbol{T}}_{j})-\tilde{\boldsymbol{T}}_{j}) \right] \notag\\
& \quad\quad + \underbracket[0.8pt][3pt]{\log ((2\pi)^\frac{n}{2}\det(\mathbf{\Sigma}_{\tilde{\boldsymbol{T}}_{j}})^\frac{1}{2})}_{\boldsymbol{const}} \notag\\
&\propto \mathbb{E}_{j\in\mathcal{V}}[\| r_{\mathbf{X}}(\tilde{\boldsymbol{T}}_{j}) - \tilde{\boldsymbol{T}}_{j} \|^2]
\end{align}
expanding $r_{\mathbf{X}}(\cdot)$ with the first-order Taylor series yields
\begin{equation}
r_{\mathbf{X}}(\tilde{\boldsymbol{T}}_{j}) = r_{\mathbf{X}}(\boldsymbol{T}_{j} + \epsilon) = r_{\mathbf{X}}(\boldsymbol{T}_{j}) + \nabla r_{\mathbf{X}}(\boldsymbol{T}_{j})\epsilon + o(\epsilon^\top\epsilon)
\end{equation}
and we have
\begin{align}
\mathcal{L} &= \mathbb{E}_{j\in\mathcal{V}}[\| r_{\mathbf{X}}(\boldsymbol{T}_{j}) + \nabla r_{\mathbf{X}}(\boldsymbol{T}_{j})\epsilon - (\boldsymbol{T}_{j} + \epsilon) + o(\epsilon^\top\epsilon) \|^2] \notag\\
&= \mathbb{E}_{j\in\mathcal{V}}[\| (r_{\mathbf{X}}(\boldsymbol{T}_{j}) - \boldsymbol{T}_{j}) + (\nabla r_{\mathbf{X}}(\boldsymbol{T}_{j})\epsilon - \epsilon)\|^2] + o(\sigma_{\epsilon}^2) \notag\\
&= \mathbb{E}_{j\in\mathcal{V}}[\| r_{\mathbf{X}}(\boldsymbol{T}_{j}) - \boldsymbol{T}_{j} \|^2] \notag\\
& \quad + 2 \mathbb{E}_{j\in\mathcal{V}}[\epsilon]^\top\mathbb{E}_{j\in\mathcal{V}}[(\nabla r_{\mathbf{X}}(\boldsymbol{T}_{j}) - \mathbf{I})^\top(r_{\mathbf{X}}(\boldsymbol{T}_{j}) - \boldsymbol{T}_{j})] \notag\\
& \quad + \left(\mathbb{E}_{j\in\mathcal{V}}[\| \nabla r_{\mathbf{X}}(\boldsymbol{T}_{j})\epsilon \|^2] + \mathbb{E}_{j\in\mathcal{V}}[\epsilon^\top\epsilon] \right. \notag\\
& \left. \quad - 2 \mathbb{E}_{j\in\mathcal{V}}[\epsilon]^\top\mathbb{E}_{j\in\mathcal{V}}[\nabla r_{\mathbf{X}}(\boldsymbol{T}_{j})] 
\vphantom{\mathbb{E}_{j\in\mathcal{V}}[\| \nabla r_{\mathbf{X}}(\boldsymbol{T}_{j})\epsilon \|^2] + \mathbb{E}_{j\in\mathcal{V}}[\epsilon^\top\epsilon]}\right) + o(\sigma_{\epsilon}^2) \notag\\
&= \mathbb{E}_{j\in\mathcal{V}}[\| r_{\mathbf{X}}(\boldsymbol{T}_{j}) - \boldsymbol{T}_{j} \|^2] \notag\\
&\quad+ \text{tr}(\mathbb{E}_{j\in\mathcal{V}}[\epsilon\epsilon^\top]\mathbb{E}_{j\in\mathcal{V}}[\nabla r_{\mathbf{X}}(\boldsymbol{T}_{j})^\top\nabla r_{\mathbf{X}}(\boldsymbol{T}_{j})]) \notag\\
& \quad + 2\mu_{\epsilon}^\top\mathbb{E}_{j\in\mathcal{V}}[(\nabla r_{\mathbf{X}}(\boldsymbol{T}_{j}) - \mathbf{I})^\top(r_{\mathbf{X}}(\boldsymbol{T}_{j}) - \boldsymbol{T}_{j})] \notag\\
& \quad - 2\mu_{\epsilon}^\top\mathbb{E}_{j\in\mathcal{V}}[\nabla r_{\mathbf{X}}(\boldsymbol{T}_{j})]
+ o(\sigma_{\epsilon}^2)
\end{align}
As the noise vector of each ego-graph $\mathcal{G}_i$ is a probabilistic simplex, the mean of noises over every edge in $\mathcal{G}_i$ is $1/\text{deg}(i)$. This derives the statistical mean of bandwidths on the entire graph $\mu_{\boldsymbol{\epsilon}}$ as
\begin{equation}\label{eq:noisemean}
\mu_{\boldsymbol{\epsilon}} = \frac{1}{2|\mathcal{E}|}\sum_{i\in\mathcal{V}}{\text{deg}(i)\cdot\frac{1}{\text{deg}(i)}} = \frac{n}{2|\mathcal{E}|}
\end{equation}
Therefore, $\mu_{\epsilon}\rightarrow 0$ when $n \ll 2|\mathcal{E}|$. In that case,
\usetagform{normalsize}
{\small
\begin{align}
\mathcal{L} &= \mathbb{E}_{j\in\mathcal{V}}[\| r_{\mathbf{X}}(\boldsymbol{T}_{j})\!-\!\boldsymbol{T}_{j} \|^2] + \sigma_{\epsilon}^2\text{tr}(\mathbb{E}_{j\in\mathcal{V}}[\nabla r_{\mathbf{X}}(\boldsymbol{T}_{j})^\top\nabla r_{\mathbf{X}}(\boldsymbol{T}_{j})]) + o(\sigma_{\epsilon}^2) \notag\\
&= \mathbb{E}_{j\in\mathcal{V}}[\| r_{\mathbf{X}}(\boldsymbol{T}_{j})\!-\!\boldsymbol{T}_{j} \|^2] + \sigma_{\epsilon}^2 \mathbb{E}_{j\in\mathcal{V}}[\| \nabla r_{\mathbf{X}}(\boldsymbol{T}_{j})\|^2_F] + o(\sigma_{\epsilon}^2)
\end{align}
}%
\end{proof}

\subsection{Proof of Theorem \ref{thm:gradopt}}\label{A.3}
\begin{proof}
We follow~\cite{RAETheory1} to complete the proof. From a generative perspective, one may consider the edge set of $\mathcal{G}_j$ as a sampled subset from $p(\boldsymbol{T}_j)$. Let
\begin{align}\label{eq:functionaldef}
f(\boldsymbol{T}_{j},r_{\mathbf{X}},\nabla r_{\mathbf{X}}) &:= p(\boldsymbol{T}_j)(\mathbb{E}_{j\in\mathcal{V}}[\| r_{\mathbf{X}}(\boldsymbol{T}_{j}) - \boldsymbol{T}_{j} \|^2] \notag\\
&\quad\quad + \sigma_{\epsilon}^2 \mathbb{E}_{j\in\mathcal{V}}[\| \nabla r_{\mathbf{X}}(\boldsymbol{T}_{j})\|^2_F])
\end{align}
Then the bandwidth prediction in \cref{eq:rdae} can be transformed into finding the extremum of an integral functional $\mathcal{L}(r_{\mathbf{X}})$:
\begin{equation}
r^* = \arg \min \mathcal{L}(r_{\mathbf{X}}) , \\
s.t. \ \mathcal{L}(r_{\mathbf{X}}) = \int_{\mathbb{R}^n}{f(\boldsymbol{T}_{j},r_{\mathbf{X}},\nabla r_{\mathbf{X}})\text{d}\boldsymbol{T}_j}
\end{equation}
Despite a multivariate functional, it can be split into individual components:
\begin{equation}\small
\mathcal{L}(r_{\mathbf{X}}) = \sum_{i=1}^{n}\int_{\mathbb{R}^n}{p(\boldsymbol{T}_j)\left(( r_{\mathbf{X},i}(\boldsymbol{T}_{j}) - T_{ij})^2 + \sigma_{\epsilon}^2 \sum_{k=1}^{n}{\left(\displaystyle\frac{\partial r_{\mathbf{X},i}(\boldsymbol{T}_{j})}{\partial T_{kj}} \right)^2}\right)\text{d}\boldsymbol{T}_j} \\
\end{equation}
We know by the Euler-Langrage equation that the optimal $r^*$ satisfies
\begin{equation}\label{eq:euler-langrage}
\frac{\partial f}{\partial r_{\mathbf{X}}} \bigg|_{r^*} - \frac{\text{d}}{\text{d}\boldsymbol{T}_j}\frac{\partial f}{\partial \nabla r_{\mathbf{X}}} \bigg|_{r^*} = 0
\end{equation}
By \cref{eq:functionaldef}, we have
\begin{equation}\label{eq:el-leftderivative}
\displaystyle\frac{\partial f}{\partial r_{\mathbf{X}}} = 2(r_{\mathbf{X},i}(\boldsymbol{T}_{j}) - T_{ij})p(\boldsymbol{T}_j),
\end{equation}
\usetagform{normalsize}
{\small
\begin{align}
\displaystyle\frac{\partial f}{\partial (\nabla r_{\mathbf{X}})_i} &= 2\sigma_{\epsilon}^2 p(\boldsymbol{T}_j) \left[\displaystyle\frac{\partial r_{\mathbf{X},k}(\boldsymbol{T}_{j})}{\partial T_{ij}} \right]_{k}^\top  \\\label{eq:el-rightderivative}
\Rightarrow \displaystyle\frac{\partial}{\partial T_{ij}}\displaystyle\frac{\partial f}{\partial (\nabla r_{\mathbf{X}})_i} &= 2\sigma_{\epsilon}^2 \left(\textstyle\frac{\partial p(\boldsymbol{T}_j)}{\partial T_{ij}} \left[\textstyle\frac{\partial r_{\mathbf{X},k}(\boldsymbol{T}_{j})}{\partial T_{ij}}\right]_{k}^\top
+ p(\boldsymbol{T}_j) \left[\textstyle\frac{\partial^2 r_{\mathbf{X},k}(\boldsymbol{T}_{j})}{\partial T_{ij}^2}\right]_{k}^\top \right)
\end{align}
}
Putting \cref{eq:el-leftderivative} and \cref{eq:el-rightderivative} into \cref{eq:euler-langrage} yields
{\small
\begin{align}\label{eq:tobesolved}
r_{\mathbf{X},k}(\boldsymbol{T}_{j}) - T_{kj} &= \frac{\sigma_{\epsilon}^2}{p(\boldsymbol{T}_{j})}\sum_{i=1}^{n}{\left(\displaystyle\frac{\partial p(\boldsymbol{T}_{j})}{\partial T_{ij}} \displaystyle\frac{\partial r_{\mathbf{X},k}(\boldsymbol{T}_{j})}{\partial T_{ij}} + p(\boldsymbol{T}_{j})\displaystyle\frac{\partial^2 r_{\mathbf{X},k}(\boldsymbol{T}_{j})}{\partial T_{ij}^2} \right)} \notag\\
&= \sigma_{\epsilon}^2 \sum_{i=1}^{n}{\left(\displaystyle\frac{\partial \log{p(\boldsymbol{T}_{j})}}{\partial T_{ij}} \displaystyle\frac{\partial r_{\mathbf{X},k}(\boldsymbol{T}_{j})}{\partial T_{ij}} + \displaystyle\frac{\partial^2 r_{\mathbf{X},k}(\boldsymbol{T}_{j})}{\partial T_{ij}^2} \right)}
\end{align}
}%
\cite{RAETheory1} gives an analytical solution of \cref{eq:tobesolved} when $\sigma_{\epsilon}^2 \rightarrow 0$:
\begin{equation}
r^*_{\mathbf{X},k}(\boldsymbol{T}_{j})\bigg|_{\sigma_{\epsilon}^2 \rightarrow 0} = T_{kj} + \sigma_{\epsilon}^2\displaystyle\frac{\partial \log{p(\boldsymbol{T}_{j})}}{\partial T_{ij}}+o(\sigma_{\epsilon}^2)
\end{equation}
so the proof concludes:
\begin{equation}
r^*_{\mathbf{X}}(\boldsymbol{T}_{j}) - \boldsymbol{T}_j \propto \nabla \log{p(\boldsymbol{T}_{j})}
\end{equation}
This indicates that the perturbed topological encoding manifold $p(\tilde{\mathbf{T}})$ is approximately equal to the original manifold $p(\mathbf{T})$ when $\boldsymbol{\epsilon}$ is small enough. Hence, despite the changing bandwidths, the optimizing objective remains invariant as the topological manifold of the original graph data.
\end{proof}

\subsection{Mildness of Assumptions}\label{A.4}

%% For preprint
% Please refer to our \href{}{preprint version}.

\subsubsection{The noise mean $\mu_{\boldsymbol{\epsilon}}$.}
Proposition \ref{thm:dae} and Theorem \ref{thm:gradopt} hold under $n \ll 2|\mathcal{E}|$, that is, the mean of bandwidths over every edge needs to be small enough (in other words, the {\it mask ratio} needs to be close to 1). According to \cref{eq:noisemean}, \texttt{Bandana}'s mask ratio is fixed as $p=1-\mu_{\boldsymbol{\epsilon}}=1-n/2|\mathcal{E}_{train}|$, which is very large in large-scale networks (and even larger in practice because the graph data is not always connected), thus the assumption can be easily satisfied. For discrete \textsc{TopoRec}s, this also implies that little information is available during training. Yet, \texttt{Bandana} keeps the global topology intact, which is conducive to mitigating the impact of the extremely high mask ratio. The mask ratios of \texttt{Bandana} throughout our experiments are listed in Table \ref{table:maskratio}, where ``Calculated'' represents the mask ratios calculated by $p =1-n/2|\mathcal{E}_{train}|$, and ``Measured'' represents the actual mask ratios measured during training.

\begin{table}
\setlength{\tabcolsep}{3pt}
\caption{\small 
Mask ratios of \texttt{Bandana} on various datasets.
}
\label{table:maskratio}
\small
\begin{center}
\adjustbox{max width=0.48\textwidth}{
\begin{tabular}{lrr}
\toprule[1.2pt]
Dataset & \multicolumn{1}{c}{Calculated} & \multicolumn{1}{c}{Measured}\\
\midrule
Cora & 0.6983 & 0.7077\\
CiteSeer & 0.5702 & 0.6048\\
PubMed & 0.7383 & 0.7571\\
Photo & 0.9622 & 0.9630\\
Computers & 0.9671 & 0.9679\\
CS & 0.8683 & 0.8697\\
Physics & 0.9182 & 0.9185\\
ogbn-arxiv & 0.9140 & 0.9158\\
ogbl-collab & 0.8840 & 0.8995\\
\bottomrule[1.2pt]
\end{tabular}}
\end{center}
\end{table}

\subsubsection{The covariance \smash{$\mathbf{\Sigma}_{\tilde{\boldsymbol{T}}_{j}}$}.}
Another prerequisite of Proposition \ref{thm:dae} and Theorem \ref{thm:gradopt} is that the covariance of \smash{$\{\tilde{\boldsymbol{T}}_{j}\}_{j=1}^{n}$} should satisfy \smash{$\mathbf{\Sigma}_{\tilde{\boldsymbol{T}}_{j}}=c\mathbf{I}$} with an arbitrary constant $c$. It is obvious that $c$ is the variance of noise $\sigma_\epsilon$, so we mainly focus on the diagonal covariance matrix, which implies the independence between every two different entries of $\tilde{\boldsymbol{T}}_{j}$. As
\usetagform{normalsize}
{\small
\begin{align}
\mathbb{E}[\tilde{T}_{ij}\tilde{T}_{kj}] &= \mathbb{E}[(T_{ij} + \epsilon_{ij})(T_{kj} + \epsilon_{kj})] \notag\\
&= \mathbb{E}[T_{ij}T_{kj} + T_{ij}\epsilon_{kj} + T_{kj}\epsilon_{ij} + \epsilon_{ij}\epsilon_{kj}] \notag\\
&= \mathbb{E}[T_{ij}T_{kj}] + \mathbb{E}[T_{ij}]\mathbb{E}[\epsilon_{kj}] + \mathbb{E}[T_{kj}]\mathbb{E}[\epsilon_{ij}] + \mathbb{E}[\epsilon_{ij}]\mathbb{E}[\epsilon_{kj}] \notag\\
&= \mathbb{E}[T_{ij}T_{kj}] + \mathbb{E}[T_{ij} + \epsilon_{ij}]\mathbb{E}[T_{kj} + \epsilon_{kj}] - \mathbb{E}[T_{ij}]\mathbb{E}[T_{kj}] \notag\\
&= \mathbb{E}[\tilde{T}_{ij}]\mathbb{E}[\tilde{T}_{kj}] + \mathbb{E}[T_{ij}T_{kj}] - \mathbb{E}[T_{ij}]\mathbb{E}[T_{kj}]
\end{align}
}
for any $i, k \in \mathcal{N}_j, i \ne k$, it is equivalent to the independence of the local topology \smash{$\{T_{ij}\}_{i \in \mathcal{N}_j}$} of every node $j$, i.e. every two incoming edges of $j$ should be independent. 
While node relationships in real-world networks are more likely to be correlated, this assumption is introduced for the brevity of the mathematical derivation of Proposition \ref{thm:dae} and Theorem \ref{thm:gradopt}. Whether they still hold without this assumption necessitates further mathematical analysis.

\begin{figure*}
    \centering
    \includegraphics[width=1.76\columnwidth]{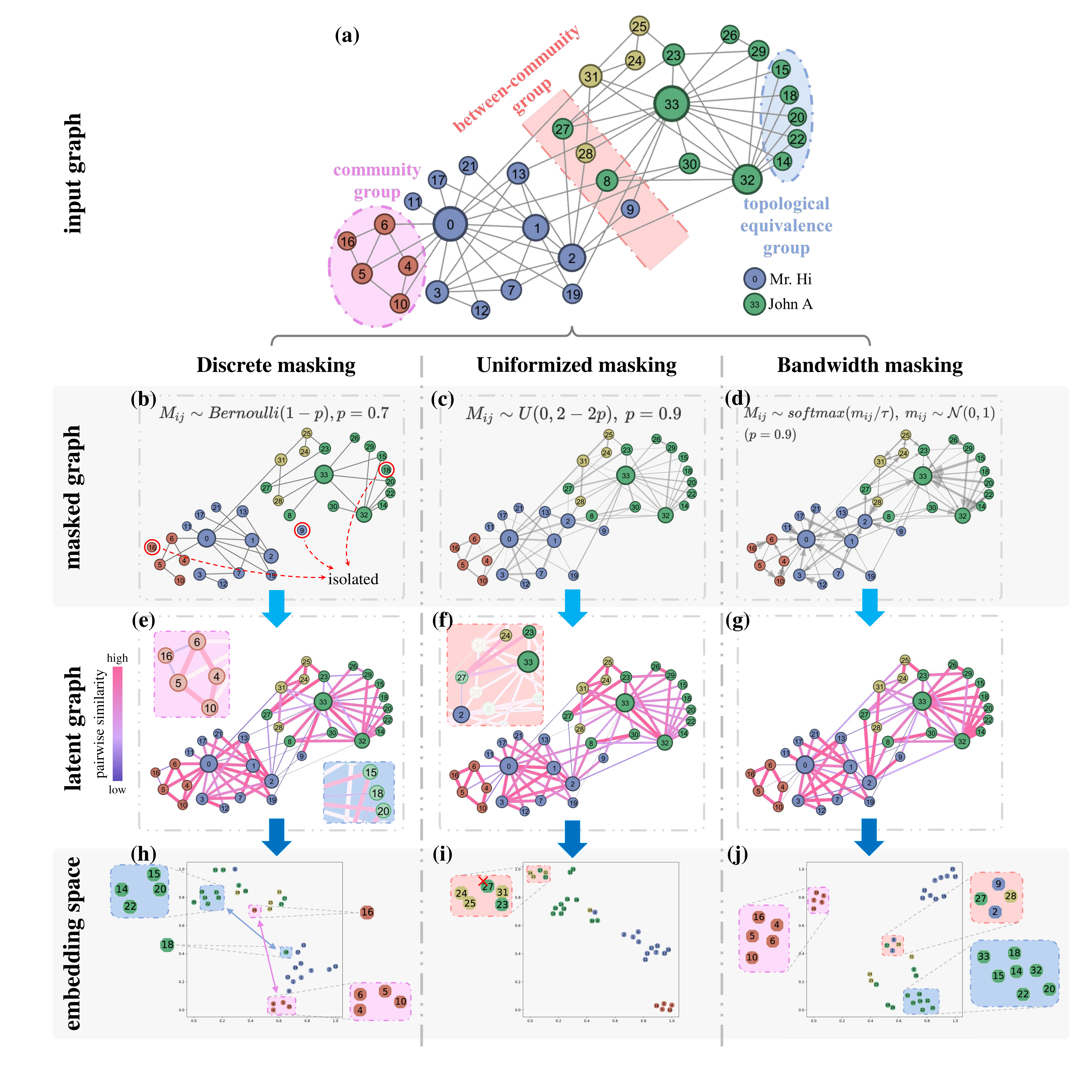}
    \caption{
    Node embedding visualization of different masking strategies. 
    \textmd{{\bf (a)} Karate Club, where each node is labeled by 1 of 4 classes (denoted by node colors). {\bf (b-d)} Edges are masked by three types of masking strategies. {\bf (e-g)} Latent graphs visualized by pairwise similarities of node embeddings. {\bf (h-j)} The 2-dimensional embedding space constructed by t-SNE~\protect\cite{t-SNE}. Colors of different insets indicate three different node groups: community group (\textcolor{magenta}{magenta}), between-community group (\textcolor{red}{red}), and topological equivalence group (\textcolor{blue}{blue}). }
    }
    \label{fig:karateclub}
\end{figure*}

\section{More Visualizations}\label{B}

In this section, we illustrate our embedding visualization study in Figure~\ref{fig:karateclub} to support the global and local informativeness of \texttt{Bandana} in Subsection \ref{4.2.1}.
We conduct the experiment on Karate Club~\cite{KarateClub}, a widely used real-world network dataset containing 34 nodes (numbered 0--33 in in Figure~\ref{fig:karateclub}(a)) and 156 undirected edges. Each node is assigned by a column of the identity matrix $\mathbf{I}_n$ as features. 
\texttt{Bandana} is trained on Karate Club with three masking setups: discrete masking (left column), uniformized masking (center column), and bandwidth masking (right column). Configurations of these setups are the same as those used in our ablation studies (Subsection~\ref{5.5}). Afterwards, we visualize the node embeddings in a 2-dimensional embedding space by t-SNE~\cite{t-SNE}, as in Figure~\ref{fig:karateclub}(h-j).
We mainly focus on three types of node groups, as shown in Figure~\ref{fig:karateclub}(a):
\begin{itemize}[leftmargin=*]
    \item \textcolor{magenta}{{\it Community group}}, a conspicuous cluster adjacent to node 0. Members include 4, 5, 6, 10, and 16.
    \item \textcolor{red}{{\it Between-community group}}, consisting of nodes located between both two clubs founded by the hub node 0 (Mr. Hi) and 33 (John A). They form two main clusters by influencing their neighboring nodes. However, between-community group members like 8, 9, 27, and 28 are similarly affected by both clubs.
    \item \textcolor{blue}{{\it Topological equivalence group}}, consisting of nodes that share their neighborhood: every member (14, 15, 18, 20, 22) is connected and only connected to node 32 and 33.
\end{itemize}

\emph{Result analysis.} According to the left column (Figure~\ref{fig:karateclub}(b), (e), (h)), discrete masking firstly undermines the graph topology by creating isolated nodes. For example, node 16 and 18 behave abnormally in the latent graph: as the message flows from 16 are blocked, it is detached from the community group, and finally drifts away from the community group cluster in the embedding space. On the other hand, the topological equivalence of 18 is not preserved due to the masking, causing it to be far away from other topological equivalence group members in the embedding space. These observations indicate that {\it discrete edge masking obstructs the message flows that may be crucial to the sink nodes}.

According to the center column (Figure~\ref{fig:karateclub}(c), (f), (i)), uniformized masking has partially resolved the isolation issues above, with 16 and 18 joined clusters they belong to. This verifies the {\it global informativeness} of non-discrete masking. However, masks sampled from a uniform distribution do not provide local informativeness. With no effort in distinguishing neighbors, node 2 is overlooked by its neighboring node 27 and is pulled towards the club of node 33, resulting in the deviation from the between-community group.

Finally, the right column (Figure~\ref{fig:karateclub}(d), (g), (j)) tells that bandwidth masking preserves the node relations of all three groups, so we have verified both the {\it global informativeness} and the {\it local informativeness} of bandwidth masking.

\section{More Discussions}\label{C}

%% For preprint
% Please refer to our \href{}{preprint version}.

This section discusses connections between \texttt{Bandana} and other deep learning models, including graph attention models, score-based models, and energy-based models. Our discussions shed light on the reliability of \texttt{Bandana} from different perspectives, and we hope they will lead to deeper insights in the future.

\subsection{Graph Attention Models}\label{C.1}
The way we use bandwidth to do weighted message propagation is inspired by the graph self-attention mechanism~\cite{GAT, SuperGAT}.%\cite{GATv2, Graphormer}
Existing studies have pointed out that attention weights should be able to distinguish different edges~\cite{GATRetrospective}, and softmax-based attention can amplify the dispersion of attention weights to be more discriminative~\cite{cosFormer}. Therefore, we hold that it is beneficial to generate bandwidth values from a softmax-amplified distribution, i.e. the Boltzmann-Gibbs distribution. Yet, our bandwidth masking and graph attention mechanisms are {\it fundamentally different}. Graph attention models empirically {\it fit} a locally optimal weight distribution of neighborhood, in which the parameter matrices converge as training goes on. Our bandwidth masking strategy does not learn the weights, but {\it randomly generates them parameter-free}. For every iteration, each edge is randomly assigned a different bandwidth, so that different neighbors will be noticed every time to help the encoder distinguish their messages. Plus, the layer-wise masking is inspired by \texttt{SuperGAT}~\cite{SuperGAT} which, however, does not give any explanations, discussions, or even empirical results in terms of the layer-wise approach. Our work also bridges this gap.

\texttt{Bandana} is currently not able to directly pre-train GAT and Graph Transformers as it also assigns weights to every edge in message propagation. It needs to be adjusted to accommodate bandwidths and attention weights. How to improve the topological learning performance of graph attention-based networks in self-supervision now remains an interesting future work.

\subsection{Score-based \& Energy-based Models}\label{C.2}
By Theorem \ref{thm:gradopt}, bandwidth prediction is equivalent to optimizing the gradient of a log probability $\nabla\log{p(\boldsymbol{T}_{j})}$. This is also called a ``{\it score}'' in the field of generation, which can be directly estimated by Score Matching~\cite{ScoreMatching} %~\cite{NCSN}
to generate samples that match the original data distribution. Therefore, \texttt{Bandana} can be seen as an implicit score-based model that adds Gaussian noise to the graph topology and learns its score.

Energy-based models (EBMs)~\cite{EBM} perform an alternate optimization process: (\lowercase{\romannumeral 1}) {\it optimizing the output}, and (ii) {\it optimizing the energy}. (i) The forward pass (or inference) of the model $f(x; \Theta):x \mapsto y$ is viewed as finding the local minimum point $y^*$ on a manifold $E_\Theta \in \mathcal{F}$. Here $E_\Theta:\mathcal{X}\times\mathcal{Y}\rightarrow\mathbb{R}$ is a scoring function of the input-output pair $(x,y)$, judging whether the output $y$ matches $x$ the best. $\mathcal{F}$ is the function space of $E_\Theta$. $E_\Theta(x,y)$ is smaller if $y$ better matches $x$. As the learning process goes on, $y$ has an increasing tendency for minimizing $E_\Theta(x,y)$, and $y=y^*$ when the model converges. Analogous to the principle of minimum energy in thermodynamics, $E_\Theta$ is called an {\it energy function}. (\lowercase{\romannumeral 2}) The backward pass of $f(x, \Theta)$ searchs on $\mathcal{F}$ for the optimal $E_\Theta$ that meets the above conditions. As the learning process goes on, $E_\Theta(x,y)$ has an increasing tendency to assign lower energy values to more compatible $(x,y)$ pairs and vice versa.

Probabilistic discriminative models $f(\boldsymbol{x}; \Theta):\boldsymbol{x} \mapsto \hat{\boldsymbol{y}}$ based on maximum likelihood estimation can directly define the energy as its negative output logit (i.e. the unnormalized probability) because it indicates which $\hat{\boldsymbol{y}}$ matches $x$ the best. However, as a non-probabilistic model, an autoencoder cannot define the energy in this way. This issue has been solved by \cite{AEScore2, AEScore3} who state that the reconstruction of denoising autoencoders is equivalent to performing regularized score matching, and the energy function can be derived using the antiderivative of the score (for any output $\hat{y}$):
\begin{equation}\label{eq:aeenergy}
E_\Theta(\boldsymbol{X}) = \log{p(\boldsymbol{X})} = \int (r(\boldsymbol{X};\Theta) - \boldsymbol{X})\text{d}\boldsymbol{X}
\end{equation}
By Proposition \ref{thm:dae}, discrete \textsc{TopoRec}s are not denoising autoencoders and hence not EBMs in this way. On the contrary, \texttt{Bandana} can be viewed as an EBM. It can be inferred from Theorem \ref{thm:gradopt} that the manifold of $r_{\mathbf{X}}(\mathbf{T}) - \mathbf{T}$ is a gradient vector field $\mathcal{T}$, on which inference is allowed to perform gradient descent on $\mathcal{T}$ to find the output $r_{\mathbf{X}}(\mathbf{T})$ that is closest to the input $\mathbf{T}$. Hence, the following corollary holds:

\begin{tcolorbox}[breakable, pad at break=2.5mm, before skip=1mm, after skip=1mm, width=1.0\linewidth, boxsep=0mm, arc=1.5mm, left=1.5mm, right=1.5mm, top=1.5mm, bottom=1mm]
\begin{corollary}[\texttt{Bandana} is energy-based]
Let $\mathcal{T} = E_{\mathbf{X},\Theta}(\mathbf{T})$ be an energy landscape similarly defined by \cref{eq:aeenergy}. Then \texttt{Bandana}'s forward and backward passes are equivalent to implicitly performing the following optimization tasks on $\mathcal{T}$:
\begin{align}
\text{Forward pass: } r_{\mathbf{X}}(\tilde{\mathbf{T}}) &= \arg \min_{\tilde{\mathbf{T}}} E_{\mathbf{X},\Theta}(\tilde{\mathbf{T}}) \notag\\
\text{Backward pass: } E^*_{\mathbf{X},\Theta} &= \arg \min_{E \in \mathcal{F}} \mathcal{L} \notag
\end{align}
\end{corollary}
\end{tcolorbox}

Such correlation between \texttt{Bandana} and EBM provides another perspective of the reliability and flexibility of bandwidth mechanisms. We see this as a good foundation for future insights.

\section{More Configurations}\label{D}

This section details our experimental configurations for reproducibility.

\subsection{Data Statistics}\label{D.1}

We use a total of 12 real-world datasets:
\begin{itemize}[leftmargin=*]
    \item {\it Citation networks}: Cora, CiteSeer, PubMed~\cite{Planetoid};
    \item {\it Co-purchase networks}: (Amazon-)Photo, (Amazon-)Computers~\cite{Amazon/Coauthor};
    \item {\it Co-author networks}: CS, Physics~\cite{Amazon/Coauthor};
    \item {\it OGB networks for node classification}: ogbn-arxiv, ogbn-mag~\cite{OGB};
    \item {\it OGB networks for link prediction}: ogbl-collab, ogbl-ppa~\cite{OGB};
    \item {\it Hyperlink networks} (directed): Wiki-CS~\cite{Wiki-CS}.
\end{itemize}
Please refer to the corresponding papers for a detailed introduction. Statistics of all datasets are listed in Table \ref{table:data}, where ``density'' stands for the percentage of all potential connections in a network that are actually positive edges, formally $\rho=\frac{|\mathcal{E}|}{n(n-1)}$.

\begin{table}
\setlength{\tabcolsep}{3pt}
\caption{\small 
Dataset statistics. \textmd{``$\dagger$'' marks the synthetic ones.}
}
\label{table:data}
\small
\begin{center}
\adjustbox{max width=0.5\textwidth}{
\begin{tabular}{lrrrrr}
\toprule[1.2pt]
Dataset & \multicolumn{1}{c}{\#nodes} & \multicolumn{1}{c}{\#edges} & \multicolumn{1}{c}{\#features} & \multicolumn{1}{c}{\#classes} & \multicolumn{1}{c}{Density (‰)} \\
\midrule
Swiss Roll$^\dagger$ & 500 & 6,712 & -- & -- & 26.9\\
Two-moon$^\dagger$ & 2,000 & 12,264 & -- & -- & 3.07\\
Cora & 2,708 & 10,556& 1,433 & 7 &1.44\\
CiteSeer&3,327&9,104&3,703&6&0.82 \\
PubMed&19,717&88,648&500&3&0.23\\
Photo&7,487&119,043&745&8&4.07\\
Computers&13,381&245,778&767&10&2.60\\
CS&18,333&81,894&6,805&15&0.24\\
Physics&34,493&247,962&8,415&5&0.21\\
Wiki-CS&11,701&216,123&300&10&1.58\\
ogbn-arxiv&169,343&2,315,598&128&40&0.08\\
ogbn-mag&736,389&10,792,672&128&349&0.02\\
ogbl-collab&235,868&2,570,930&128& -- &0.05\\
ogbl-ppa&576,289&30,326,273&58& -- &0.09\\
\bottomrule[1.2pt]
\end{tabular}}
\end{center}
\end{table}

\subsection{Hardware \& Environments}\label{D.2}

%% For preprint
% Please refer to our \href{}{preprint version}.

\texttt{Bandana} is built upon PyTorch~\cite{PyTorch} 1.12.1 and PyTorch Geometric (PyG)~\cite{PyG} 2.3.1. The latter provides all 7 datasets used throughout the quantitative experiments except ogbn-arxiv and ogbl-collab, which are from the OGB 1.3.5 package~\cite{OGB}. Two synthetic datasets used in Section \ref{5.2} come from the PyGSP package~\cite{PyGSP}. All experiments are conducted on a 24GB NVIDIA GeForce GTX 3090 GPU with CUDA 11.3.

\begin{table*}
\caption{\small 
Detailed hyperparameters of \texttt{Bandana}.
}
\label{table:hyperparameter}
\small
\begin{center}
\begin{tabular}{lcccccccc}
\toprule[1.2pt]
Dataset & Cora & CiteSeer & PubMed & Photo  & Computers & CS   & Physics & ogbn-arxiv\\
\midrule
No. of layers  & 3 & 5 & 2 & 2 & 3 & 2 & 2 & 4 \\
Learning rate $\gamma$ & 1e-2 & 2e-2     & 1e-3   & 2e-3 & 1e-3  & 1e-2 & 2e-3  & 5e-4\\
Bandwidth temperature $\tau$ & 0.9  & 0.2    & 0.2    & 1  & 0.4       & 1e-6 & 0.4     & 0.4 \\
Intermediate feature dim.         & 256  & 256      & 64     & 256 & 256       & 64   & 256     & 256\\
Output feature dim. & 256  & 256      & 32     & 64 & 64        & 32   & 128     & 256\\
Encoder dropout & 0.8  & 0.8      & 0.6    & 0.8 & 0.5       & 0.8  & 0.8     & 0.2\\
Decoder dropout & 0  & 0        & 0.7    & 0.2   & 0.2       & 0.2  & 0.2     & 0\\
Weight decay (for encoder)        & 5e-5 & 5e-5     & 5e-5   & 5e-5 & 0         & 5e-5 & 5e-5    & 5e-5       \\
Weight decay (for linear probing) & 5e-3 & 1e-1     & 5e-5   & 5e-4 & 5e-4      & 1e-3 & 1e-3    & 1e-4 \\
\bottomrule[1.2pt]
\end{tabular}
\end{center}
\end{table*}

\subsection{Model Setup \& Hyperparameters}\label{D.3}

\emph{Training setup.}
We follow the train/validation/test split of previous work~\cite{MaskGAE}. To be specific, we use all existing official splits. For all datasets, edge sets are divided into $\mathcal{E}_\text{train}\!:\!\mathcal{E}_\text{val}\!:\!\mathcal{E}_\text{test} = 85\%\!:\!5\%\!:\!10\%$ for training and the downstream link prediction. As for node classification, the official split of Planetoid and ogbn-arxiv is adopted and node sets of other datasets are divided into $\mathcal{V}_\text{train}\!:\!\mathcal{V}_\text{val}\!:\!\mathcal{V}_\text{test} = 10\%\!:\!10\%\!:\!80\% $.

\texttt{Bandana} employs a GCN encoder ($\tilde{\mathbf{G}}=\boldsymbol{\Sigma}_e\hat{\tilde{\mathbf{D}}}^{-1/2}\hat{\tilde{\mathbf{A}}}\hat{\tilde{\mathbf{D}}}^{-1/2}$ in \cref{eq:bandanaprop}) with 1 to 5 layers and a fixed 2-layer MLP decoder with dropout. For brevity, \texttt{Bandana} does not resort to extra techniques such as path masking, degree regression~\cite{MaskGAE}, cross-correlation decoding~\cite{S2GAE}, re-masking, and random feature substitution~\cite{GraphMAE}. All non-linear layers (of the encoder, decoder, and learnable downstream branches) are Xavier-initialized with biases 0. Every layer of the encoder is equipped with batch normalization~\cite{BN}, dropout, and an ELU activation function~\cite{ELU}.
We perform grid search for the learning rate $\gamma$ and temperature $\tau$ over the searching space \{5e-4, 1e-3, 2e-3, 5e-3, 1e-2, 2e-2\} and \{1e-6, 0.1, 0.2, 0.3, ..., 1\} respectively. Adam~\cite{Adam} is used as the model optimizer. For all datasets except ogbn-arxiv, we use a fixed training strategy of 1000 epochs with early stopping, the patience of which is set to 30. For ogbn-arxiv, 100 epochs with batch size $2^{16}$. As with previous work, both grid search and early stopping are carried out on the validation set, and the best validation models are saved for testing.

\emph{Linear probing for node classification.} The so-called linear probing~\cite{MAE} first performs unsupervised pre-training on both the encoder and decoder. Then, the decoder is substituted with a Xavier-initialized linear layer. It is trained (with the encoder frozen) for another 100 epochs with a fixed learning rate of 1e-2 to obtain the classification logits. 

All model hyperparameters for node classification are given in Table \ref{table:hyperparameter}. Please refer to the configuration in our source code for link prediction. They are selected manually on the validation set among several candidate values, except the grid-searched parameters.

\section{More Experiment Analyses}\label{E}

This section showcases the rest of our experimental results to answer some additional research questions:
\begin{itemize}[leftmargin=*]
    \item ARQ1. {\it What is the time and space consumption of \texttt{Bandana}?}
    \item ARQ2. {\it Why is the dot-product probing setup fairer for evaluating self-supervised models?}
    \item ARQ3. {\it What is the performance of \texttt{Bandana} on more larger-scale datasets except ogbn-arxiv?}
    \item ARQ4. {\it Is \texttt{Bandana} able to be generalized to different graph types, e.g. directed graphs?}
    \item ARQ5. {\it How do certain parameters affect \texttt{Bandana}'s performance?}
    \item ARQ6. {\it How does \texttt{Bandana} perform with limited training data?}
\end{itemize}

\subsection{Time and Space Consumptions (ARQ1)}\label{E.1}

In this subsection, we provide the time and space complexity of \texttt{Bandana} as well as the representative baselines.

\emph{Complexity analysis.}
The time and space complexity of each component of \texttt{Bandana} is listed in Table \ref{table:complexity}. % Note that we incorporate the bandwidths into the sparse matrix multiplication operator of PyTorch Geometric~\cite{PyG} in the implementation of \texttt{Bandana}. Thus, our weighted message propagation shares the time and space complexity with the original sparse message passing. 
The total time and space complexity tells that {\it \texttt{Bandana} scales linearly w.r.t. the number of nodes and edges}. Despite that \texttt{Bandana} requires extra time for continuous masking and layer-wise decoding, the overall time complexity of \texttt{Bandana} is at the same level as \texttt{MaskGAE}. Beyond that, the lightweight decoder architecture reduces the difference in time consumption to a relatively small extent in practice.

\emph{Quantitative comparison.} We have further measured the wall-clock pre-training time and peak GPU memory consumption for \texttt{Bandana} as well as the baselines above. All self-supervised models are used to pre-train a 2-layer GCN for fixed 500 epochs (100 for ogbn-arxiv). 
As shown in Table \ref{table:timespace}, \texttt{Bandana} is not as efficient as \texttt{MaskGAE} because it utilizes the whole training graph for message propagation and the layer-wise prediction multiplies the decoding time complexity by $K$. However, the lightweight decoder architecture makes the overall time gap much smaller than $K$ times (1.7$\times$ at most), making the time-effectiveness tradeoff more justifiable. In terms of GPU memory consumption, \texttt{Bandana} is overall on par with \texttt{MaskGAE}. The reason is that 
discrete edge masking actually does not offer remarkable memory savings, as $\mathcal{O}(\vert\mathcal{E}\vert)$ is nearly negligible in the total memory cost. Therefore, \texttt{Bandana} still maintains the lightweight advantage of masked graph autoencoders.
% , which is faster and consumes less memory than the $\mathcal{O}(\vert\mathcal{V}\vert^2)$ level methods such as \texttt{SeeGera} and \texttt{GCA}, especially on large-scale networks.

\begin{table}
\caption{\small 
Time \& space complexity of \texttt{Bandana}'s components. }
\label{table:complexity}
\small
\begin{center}
\adjustbox{max width=\columnwidth}{
\begin{tabular}{lll}
\toprule[1.2pt]
Component & Time & Space \\
\midrule
Bandwidth masking & $\mathcal{O}(\vert\mathcal{E}\vert)$ & $\mathcal{O}(\vert\mathcal{E}\vert)$ \\
Encoding & $\mathcal{O}(Kd\vert\mathcal{E}\vert+Kd^2\vert\mathcal{V}\vert)$ & $\mathcal{O}(Kd^2+\vert\mathcal{E}\vert+Kd\vert\mathcal{V}\vert)$ \\
Decoding & $\mathcal{O}(Kd^2\vert\mathcal{E}\vert)$ & $\mathcal{O}(d\vert\mathcal{E}\vert)$ \\
Bandwidth prediction & $\mathcal{O}(K\vert\mathcal{E}\vert)$ & $\mathcal{O}(\vert\mathcal{E}\vert)$ \\
Total & $\mathcal{O}(Kd^2\vert\mathcal{E}\vert+Kd^2\vert\mathcal{V}\vert)$ & $\mathcal{O}(Kd^2+d\vert\mathcal{E}\vert+Kd\vert\mathcal{V}\vert)$ \\
\bottomrule[1.2pt]
\end{tabular}}
\end{center}
\end{table}

\begin{table}
\caption{\small 
Comparison of pre-training time (seconds) and memory consumption (MB). \textmd{``OOM'' stands for ``Out-Of-Memory'' on a 24GB GPU. }
}
\label{table:timespace}
\small
\begin{center}
\adjustbox{max width=\columnwidth}{
\begin{tabular}{lrrrrr}
\toprule[1.2pt]
Dataset & \multicolumn{1}{c}{\texttt{VGAE}} & \multicolumn{1}{c}{\texttt{GCA}} & \multicolumn{1}{c}{\texttt{SeeGera}} & \multicolumn{1}{c}{\texttt{MaskGAE}} & \multicolumn{1}{c}{\texttt{Bandana}} \\
\midrule
\multicolumn{6}{c}{Pre-training time (s)} \\
\midrule
Cora & 1.96e01 & 4.19e01 & 1.25e02 & 1.16e01 & 1.95e01 \\
CiteSeer & 2.44e01 & 4.53e01 & 4.59e02 & 1.40e01 & 2.08e01 \\
PubMed & 4.06e01 & 2.29e02 & 9.01e02 & 5.68e01 & 7.62e01 \\
ogbn-arxiv & 4.79e02 & OOM & OOM & 2.47e02 & 4.07e02 \\
ogbl-collab & 1.25e03 & OOM & OOM & 6.39e02 & 9.47e02 \\
\midrule
\multicolumn{6}{c}{Peak GPU memory (MB)} \\
\midrule
Cora & 1,169 & 1,431 & 2,045 & 1,171 & 1,459 \\
CiteSeer & 1,235 & 1,687 & 2,959 & 1,243 & 1,677 \\
PubMed  & 1,409 & 12,267 & 20,781 & 1,345 & 1,523 \\
ogbn-arxiv & 8,941 & >24,576 & >24,576 & 6,149 & 4,925 \\
ogbl-collab & 7,289 & >24,576 & >24,576 & 5,283 & 6,269 \\
\bottomrule[1.2pt]
\end{tabular}}
\end{center}
\end{table}

\begin{table*}
\caption{\small 
Average AUC (\%) of link prediction under the end-to-end training/fine-tuning (ETE/FT) and the dot-product probing (DPP).
}
\label{table:dpp}
\small
\begin{center}
\adjustbox{max width=\textwidth}{
\begin{tabular}{lclllllll}
\toprule[1.2pt]
Model & Setup & \multicolumn{1}{c}{Cora} & \multicolumn{1}{c}{CiteSeer} & \multicolumn{1}{c}{PubMed} & \multicolumn{1}{c}{Photo} & \multicolumn{1}{c}{Computers} & \multicolumn{1}{c}{CS} & \multicolumn{1}{c}{Physics} \\[-0.1em]
\midrule
\multirow{2}{*}{\texttt{T-BGRL}~\cite{T-BGRL}} & FT & 91.34 & 95.70 & 95.70 & 98.22 & 97.76 & 95.91 & 96.42\\ [-0.2em]
& DPP & 73.18 (\textcolor{red!60!black}{\bf $\downarrow$18.2}) & 78.11 (\textcolor{red!60!black}{\bf $\downarrow$17.6}) & 76.21 (\textcolor{red!60!black}{\bf $\downarrow$19.5}) & 80.80 (\textcolor{red!60!black}{\bf $\downarrow$17.9}) & 84.60 (\textcolor{red!60!black}{\bf $\downarrow$13.2}) & 70.08 (\textcolor{red!60!black}{\bf $\downarrow$25.8}) & 89.18 (\textcolor{red}{$\downarrow$7.24}) \\ [-0.2em]
\arrayrulecolor[RGB]{180, 180, 180}\cmidrule(){2-9}\arrayrulecolor[RGB]{0, 0, 0} 
\multirow{2}{*}{\texttt{S2GAE}~\cite{S2GAE}} & ETE & 93.41 & 93.14 & 98.34 & 96.97 & 95.49 & 94.13 & 97.02\\ [-0.2em]
& DPP & 89.27 (\textcolor{red}{$\downarrow$4.14}) & 86.35 (\textcolor{red}{$\downarrow$6.79}) & 89.53 (\textcolor{red}{$\downarrow$8.81}) & 86.80 (\textcolor{red!60!black}{\bf $\downarrow$10.2}) & 84.16 (\textcolor{red!60!black}{\bf $\downarrow$11.3}) & 86.60 (\textcolor{red}{$\downarrow$7.53}) & 88.92 (\textcolor{red}{$\downarrow$8.10}) \\ [-0.2em]
\arrayrulecolor[RGB]{180, 180, 180}\cmidrule(){2-9}\arrayrulecolor[RGB]{0, 0, 0} 
\multirow{2}{*}{\texttt{MaskGAE}-edge~\cite{MaskGAE}} & ETE & 96.46 & 97.91 & 98.84 & 98.73 & 98.72 & 98.92 & 95.10\\ [-0.2em]
& DPP & 95.66 (\textcolor{red!60!white}{$\downarrow$0.80}) & 97.02 (\textcolor{red!60!white}{$\downarrow$0.89}) & 96.51 (\textcolor{red}{$\downarrow$2.33}) & 81.12 (\textcolor{red!60!black}{\bf $\downarrow$17.6}) & 76.23 (\textcolor{red!60!black}{\bf $\downarrow$22.5}) & 96.50 (\textcolor{red}{$\downarrow$2.42}) & 93.09 (\textcolor{red}{$\downarrow$2.01})\\ [-0.2em]
\arrayrulecolor[RGB]{180, 180, 180}\cmidrule(){2-9}\arrayrulecolor[RGB]{0, 0, 0} 
\multirow{2}{*}{\texttt{MaskGAE}-path~\cite{MaskGAE}} & ETE & 96.43 & 97.92 & 98.74 & 98.56 & 98.73 & 98.72 & 98.76\\ [-0.2em]
& DPP & 95.47 (\textcolor{red!60!white}{$\downarrow$0.96}) & 97.21 (\textcolor{red!60!white}{$\downarrow$0.71}) & 97.19 (\textcolor{red}{$\downarrow$1.55}) & 80.46 (\textcolor{red!60!black}{\bf $\downarrow$18.1}) & 73.24 (\textcolor{red!60!black}{\bf $\downarrow$25.5}) & 92.26 (\textcolor{red}{$\downarrow$6.46}) & 94.00 (\textcolor{red}{$\downarrow$4.76})\\ [-0.2em]
\arrayrulecolor[RGB]{180, 180, 180}\cmidrule(){2-9}\arrayrulecolor[RGB]{0, 0, 0} 
\rowcolor{gray!25}
& ETE & 95.84 & 97.49 & 97.32 & 97.61 & 96.38 & 98.50 & 98.53\\ [-0.2em]
\rowcolor{gray!25}
\multirow{-2}{*}{\texttt{Bandana}} & DPP & 95.71 (\textcolor{red!60!white}{$\downarrow$0.13}) & 96.89 
 (\textcolor{red}{$\downarrow$1.08}) & 97.26 (\textcolor{red!60!white}{$\downarrow$0.06}) & 97.24 (\textcolor{red!60!white}{$\downarrow$0.37}) & 97.33 (\textcolor{green!60!black}{$\uparrow$0.95}) & 97.42 (\textcolor{red}{$\downarrow$1.08}) & 97.02 (\textcolor{red}{$\downarrow$1.51})\\ [-0.2em]
\bottomrule[1.2pt]
\end{tabular}}
\end{center}
\end{table*}

\subsection{The Dot-Product Probing (ARQ2)}\label{E.2}

\begin{figure}
    \centering
    \includegraphics[width=0.8\columnwidth]{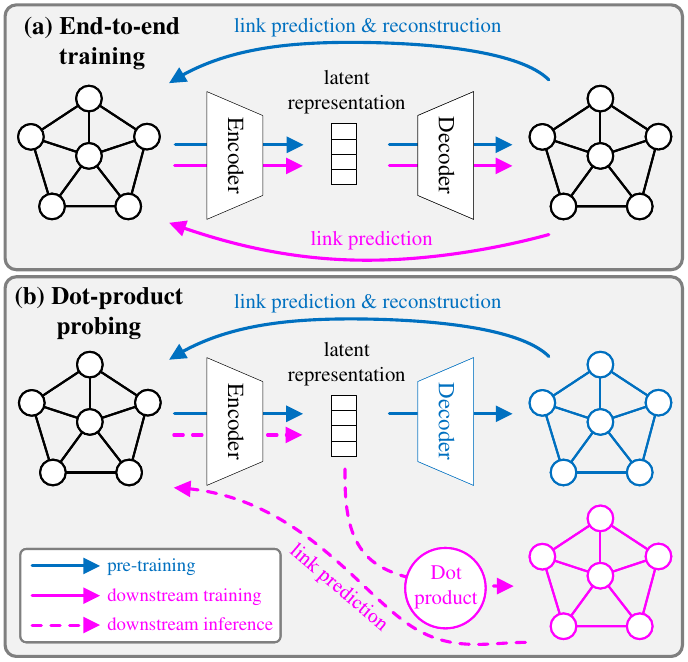}
    \caption{
    The workflow of end-to-end training (ETE) and dot-product probing (DPP).
    }
    \label{fig:dpp}
\end{figure}

In this subsection, we reiterate the necessity of dot-product probing in self-supervised link prediction. 
We first reveal three shortcomings of the traditional end-to-end evaluation scheme. Firstly, ETE has become more like a fully supervised case than self-supervised, which is not effective enough to evaluate the generalization ability of self-supervised methods. Secondly, ETE relies on the trained decoder parameters in downstream training, leading to the over-optimistic evaluation results and saturation of the link prediction accuracy on many datasets~\cite{HeaRT}. Lastly, it is unfair to methods that do not learn by link prediction, as downstream information is inappropriately provided in the pre-training phase. In contrast, the dot-product probing simply replaces the decoder with a dot-product decoder during the evaluation, illustrated in Figure \ref{fig:dpp}. DPP solely takes the latent representations provided by the encoder and directly gets the link prediction results, without any additional downstream training process, which decouples the pre-training and downstream adaptation and thus provides a fairer evaluation of self-supervised methods.

We consider several baselines specifically designed for link prediction, namely 
\texttt{T-BGRL}, \texttt{S2GAE}, and \texttt{MaskGAE}. Their link prediction performance is compared by two evaluation strategies: end-to-end training or fine-tuning (ETE/FT), and dot-product probing (DPP). For ETE/FT, \texttt{T-BGRL} fine-tunes with a 1-layer Hadamard-product MLP decoder, while the others are trained end-to-end with a 2-layer MLP decoder. Table \ref{table:dpp} shows that, under the DPP setting, the link prediction accuracies of all baselines on a vast majority of datasets show a noticeable decrease compared to those under the ETE/FT setup. As a contrastive method, \texttt{T-BGRL} is more dependent on the fine-tuning process for link prediction, so its performance is the most vulnerable out of all baselines. However, the same phenomenon is observed on the \textsc{TopoRec}s as well. The performance of \texttt{MaskGAE} even decreases by more than 17\% and 22\% on Photo and Computers, respectively. 
The clear sensitivity to the DPP setting suggests that decoders of the baseline models play a major role in their excellent performance. \texttt{Bandana}, however, is relatively immune to DPP, indicating that its effectiveness is provided by the encoder and thus more generalizable.

\setlength{\fboxsep}{0.5pt}
\setlength{\abovecaptionskip}{5pt}
\begin{table*}
\centering
\caption{\small 
Additional performance on OGB datasets. \textmd{Best results in each column are in {\bf bold}.}}
\label{table:extraogb}
\subfloat[
\textbf{Hits@k(\%) of link prediction on ogbl-collab.}
\label{table:collab}
]{
\centering
\begin{minipage}{0.3\linewidth}{\begin{center}
\adjustbox{max width=1\linewidth}{
\begin{tabular}{lcc}
\toprule 
Model & Hits@20 & Hits@50 \\
\midrule 
\texttt{MaskGAE}-edge & 58.59 ± 1.19 & 65.58 ± 0.43 \\
\texttt{MaskGAE}-path & 58.43 ± 1.06 & 65.58 ± 0.58 \\
\rowcolor{gray!25}
\texttt{Bandana} & {\bf 60.42 ± 0.84} & {\bf 67.77 ± 0.72} \\
\bottomrule 
\end{tabular}}
\end{center}}\end{minipage}
}
\hspace{1em}
\subfloat[
\textbf{Hits@k(\%) of link prediction on ogbl-ppa} (by running 3 times).
\label{table:ppa}
]{
\begin{minipage}{0.3\linewidth}{\begin{center}
\adjustbox{max width=1\linewidth}{
\begin{tabular}{lcc}
\toprule 
Model & Hits@50 & Hits@100 \\
\midrule 
\texttt{MaskGAE}-edge & 13.16 ± 0.27 & 26.78 ± 2.44 \\
\texttt{MaskGAE}-path & 10.72 ± 0.73 & 24.32 ± 1.78 \\
\rowcolor{gray!25}
\texttt{Bandana} & {\bf 30.24 ± 0.90} & {\bf 42.05 ± 0.97} \\
\bottomrule 
\end{tabular}}
\end{center}}\end{minipage}
}
\hspace{1em}
\subfloat[
\textbf{Micro-F1(\%) and macro-F1(\%) of node classification on ogbn-mag.}
\label{table:mag}
]{
\begin{minipage}{0.3\linewidth}{\begin{center}
\adjustbox{max width=1\linewidth}{
\begin{tabular}{lcc}
\toprule 
Model & Micro-F1 & Macro-F1 \\
\midrule 
\texttt{MaskGAE}-edge & 32.87 ± 0.36 & 13.94 ± 0.16 \\
\texttt{MaskGAE}-path & 32.86 ± 0.34 & 14.13 ± 0.29 \\
\rowcolor{gray!25}
\texttt{Bandana} & {\bf 33.08 ± 0.46} & {\bf 14.49 ± 0.36} \\
\bottomrule 
\end{tabular}}
\end{center}}\end{minipage}
}
\end{table*}

\subsection{Experiments on OGB Datasets (ARQ3)}\label{E.3}

Here we provide the additional link prediction results on ogbl-collab\footnote{We only keep the papers published in 2010 and beyond for training.} and ogbl-ppa, as well as node classification on ogbn-mag\footnote{We only keep the ``paper-cites-paper'' relation for training.}. We keep the end-to-end training on ogbl-collab and ogbl-ppa (on which we observe that almost all models fail with dot-product probing, as it may be too hard to achieve on large-scale datasets). 
We report results of ogbl-collab, ogbl-ppa, and ogbn-mag in Table \ref{table:extraogb}(a-c) respectively. \texttt{Bandana} is observed to consistently outperform both configurations of \texttt{MaskGAE} on all of the OGB datasets, indicating its advantage of topological learning on large-scale networks.

\begin{table}
\caption{\small 
Micro-F1(\%) and Macro-F1(\%) of node classification on Wiki-CS. \textmd{Best results in each column are in {\bf bold}.}}
\label{table:wikics}
\begin{center}
\begin{tabular}{lcc}
\toprule 
Model & Micro-F1 & Macro-F1 \\
\midrule 
\texttt{GCA} & 78.19 ± 0.05 & 75.30 ± 0.08 \\
\rowcolor{gray!25}
\texttt{Bandana} & {\bf 78.85 ± 0.25} & {\bf 75.47 ± 0.35} \\
\bottomrule 
\end{tabular}
\end{center}
\end{table}

\begin{table}
\caption{\small 
Node classification accuracy(\%) under semi-supervised setting.}
\label{table:semi}
\begin{center}
\adjustbox{max width=1\linewidth}{
\begin{tabular}{llccc}
\toprule 
\multirow{2}{*}{Dataset} & \multirow{2}{*}{Model} & \multicolumn{3}{c}{Training data ratio} \\
\arrayrulecolor[RGB]{180, 180, 180}\cmidrule(){3-5}\arrayrulecolor[RGB]{0, 0, 0} 
& & 10\% (original) & 5\% & 1\% \\
\midrule 
& \texttt{MaskGAE}-edge & 92.29 ± 0.25 & 92.04 ± 0.11 & 88.61 ± 0.36 \\
\rowcolor{gray!25}
\cellcolor{white}\multirow{-2}{*}{CS} & 	
\texttt{Bandana} & 93.10 ± 0.05 & 92.85 ± 0.05 & 90.40 ± 0.24 \\
\arrayrulecolor[RGB]{180, 180, 180}\cmidrule(){2-5}\arrayrulecolor[RGB]{0, 0, 0} 
& \texttt{MaskGAE}-edge & 95.10 ± 0.04 & 94.92 ± 0.03 & 93.56 ± 0.18 \\
\rowcolor{gray!25}
\cellcolor{white}\multirow{-2}{*}{Physics} & \texttt{Bandana} & 95.57 ± 0.04 & 95.41 ± 0.04 & 94.34 ± 0.03 \\
\bottomrule 
\end{tabular}}
\end{center}
\end{table}

\subsection{Experiments on Wiki-CS (ARQ4)}\label{E.4}

Our bandwidth strategy has great potentiality in handling problems on a wide range of graph types, because bandwidth is a special kind of edge weight which is generally applicable. For example, \texttt{Bandana} is naturally compatible with directed graphs like Wiki-CS. We conduct an experiment on Wiki-CS directly on the model architecture (with 10\%:10\%:80\% data split) and show node classification results in Table \ref{table:wikics}. \texttt{Bandana} performs better than \texttt{GCA}, a contrastive framework providing native support for Wiki-CS, owing much to the local informativeness of bidirected and dispersive edge weights.

\subsection{Parameter Analyses (ARQ5)}\label{E.5}

\subsubsection{Effect of the temperature.}\label{E.5.1} As discussed in Section \ref{4.1}, the temperature $\tau$ of the Boltzmann-Gibbs distribution controls the continuity of the mask. 
% This is also called {\it temperature scaling} in the field of calibration~\cite{TS}, distillation~\cite{KD}, etc. 
Figure \ref{fig:temp} illustrates the node classification performance with different values of $\tau$, set as 1e-6, 0.1, 0.2, ..., 0.9, 1, 2, and 5. The extent of performance fluctuation w.r.t. temperature varies across datasets, as does the temperature range for the best accuracy, such as [0.8, 0.9] for Cora and [0.2, 1] for PubMed. However, it can be observed on the vast majority of datasets that the model performance declines if $\tau$ is too small (i.e. the discretized mask) or too large (i.e. the uniformized mask).

\begin{figure}
    \centering
    \includegraphics[width=1\columnwidth]{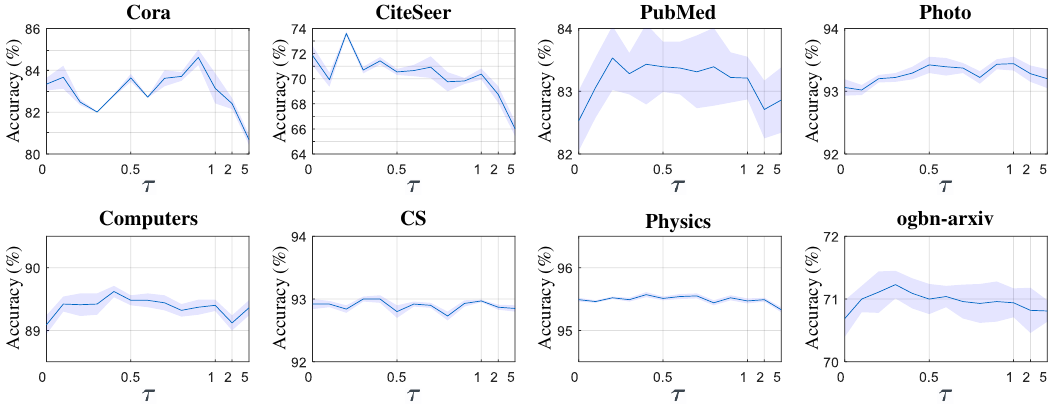}
    \caption{
    Node classification accuracy w.r.t. the temperature. 
    }
    \label{fig:temp}
\end{figure}

\begin{figure}
    \centering
    \includegraphics[width=1\columnwidth]{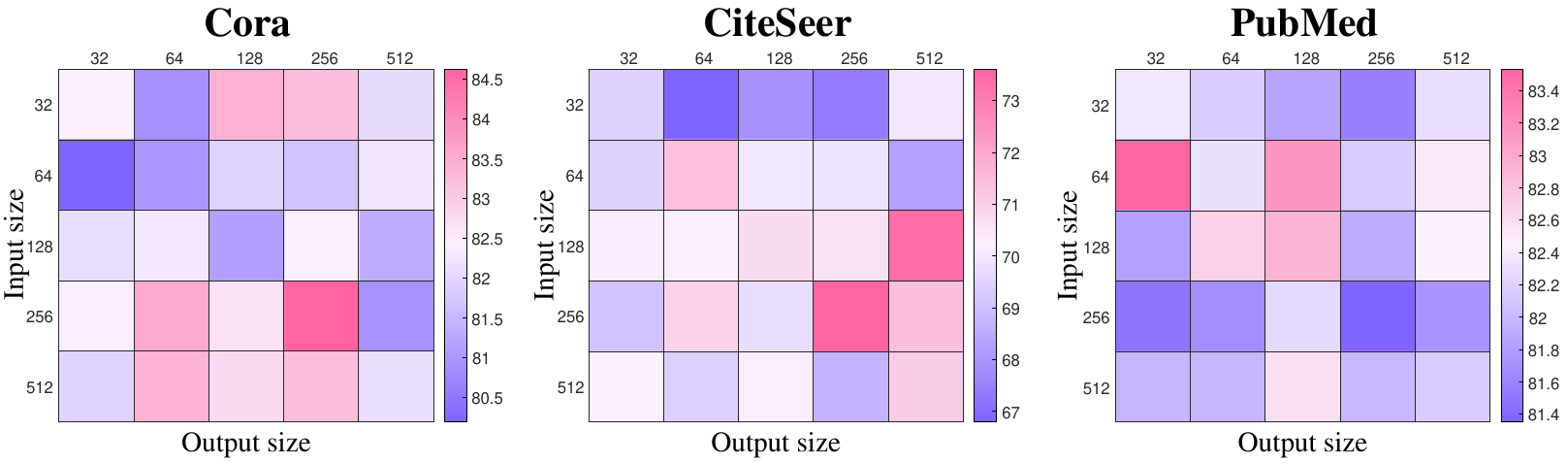}
    \caption{
    Average node classification accuracy w.r.t. the embedding size.
    }
    \label{fig:embsize}
\end{figure}

\subsubsection{Effect of embedding size.}\label{E.5.2}

We also analyze the effect of input and output embedding size of the decoder layers. Figure~\ref{fig:embsize} shows that \texttt{Bandana} does not necessarily benefit from a large embedding dimension, even if the dataset is relatively large, as the best input and output embedding size on PubMed is 64 and 32 respectively. On the other hand, CiteSeer prefers a larger embedding size because it benefits from a deeper encoder in our case, but setting a deeper encoder in a discrete \texttt{TopoRec} will instead lead to over-smoothing and a decrease in performance.

\subsection{The Semi-Supervised Setting (ARQ6)}\label{E.6}

We further investigate the performance of \texttt{Bandana} under the semi-supervised setting with scarce-labeled training data. We limit the node label ratio of CS and Physics (two datasets without official data splits) to 5\% and 1\% for downstream training and showcase the node classification accuracy in Table \ref{table:semi}. 
It's noteworthy that the performance of both \texttt{MaskGAE} and \texttt{Bandana} drops as the ratio of labeled training data decreases. However, \texttt{Bandana} still outperforms \texttt{MaskGAE} with scarce downstream supervision signals. 

\end{document}